\documentclass[twoside,11pt]{article}

\usepackage{amsthm}

\usepackage[abbrvbib,preprint]{jmlr2e}
\usepackage[T1]{fontenc}

\usepackage{bm}
\usepackage{amsmath}
\usepackage{mathrsfs}
\usepackage{amssymb}

\newcommand{\holder}[0]{H\"{o}lder}
\newcommand{\holders}[0]{H\"{o}lder's}

\newcommand{\nn}[1]{\mathbb{N}^{#1}}
\newcommand{\rr}[1]{\mathbb{R}^{#1}}
\newcommand{\rrp}[1]{\mathbb{R}_+^{#1}}

\newcommand{\zz}[1]{\mathbb{Z}^{#1}}
\newcommand{\zzp}[1]{\mathbb{Z}_+^{#1}}

\newcommand{\set}[1]{\mathcal{#1}}

\NewDocumentCommand{\eabs}{s m}{
    \IfBooleanTF{#1}
    {
        \langle {#2}\rangle
    }
    {
        \ensuremath{\left\langle {#2}\right\rangle}
    }
}

\NewDocumentCommand{\abs}{s m}{
    \IfBooleanTF{#1}
    {
        | {#2}|
    }
    {
        \ensuremath{\left| {#2}\right|}
    }
}

\NewDocumentCommand{\norm}{s m m}{
    \IfBooleanTF{#1}
    {
        \| {#2}\|
    }
    {
        \ensuremath{\left\| {#2}\right\|}
    }
    \IfValueT{#3}{_{#3}}
}

\NewDocumentCommand{\seminorm}{s m m}{
    \IfBooleanTF{#1}
    {
        | {#2}|
    }
    {
        \ensuremath{\left| {#2}\right|}
    }
    \IfValueT{#3}{_{#3}}
}

\newcommand{\innerProd}[2]{\left\langle {#1},{#2}\right\rangle}

\newcommand{\chF}[1]{\chi_{#1}}

\newcommand{\chEF}[2]{\chF{#1}^{#2}}

\NewDocumentCommand{\schwartz}{o}{\mathscr{S}
    \IfValueT{#1}{(#1)}
}
\NewDocumentCommand{\F}{}{\mathscr{F}}
\NewDocumentCommand{\FL}{s m o o}{
    \IfBooleanTF{#1}
        {
            \F L^{#2}
            \IfValueTF{#3}
                {(#3\IfValueT{#4}{;#4})}
                {\IfValueT{#4}(#4)}
        }
        {\F L^{#2}\IfValueT{#3}{_{#3}}\IfValueT{#4}{(#4)}}
    }
\NewDocumentCommand{\barron}{s m o}{
    \IfBooleanTF{#1}{\set{B}^*}{\set{B}}_{#2}\IfValueT{#3}{(#3)}
}
\NewDocumentCommand{\sobolev}{s m o o}{
    \IfBooleanTF{#1}
        {
            H^{#2}
        }
        {
            W^{#2}
        }
        \IfValueT{#3}
            {(#3\IfValueT{#4}{;#4})}
    }

\newcommand{\limto}[1]{\xrightarrow[#1]{}}

\newcommand{\dee}[0]{\mathop{\mathrm{d}\!}}
\newcommand{\sinc}[0]{\operatorname{sinc}}
\newcommand{\conv}[0]{*}

\usepackage{enumitem}
\usepackage{amsthm}

\usepackage[capitalize]{cleveref}

\newtheorem{assumption}[theorem]{Assumption}

\crefformat{equation}{#2(#1)#3}
\crefformat{enumi}{#2#1#3}

\setlist[enumerate,1]{label={$(\roman*)$},
                   ref  ={$(\roman*)$}}

\newlist{propenum}{enumerate}{1} 
\setlist[propenum]{label=$(\roman*)$, ref=\theproposition~$(\roman*)$}
\crefalias{propenumi}{proposition} 

\begin{document}

\title{Weighted Sobolev Approximation Rates for Neural Networks on Unbounded Domains}

\author{\name Ahmed Abdeljawad$^*$\email ahmed.abdeljawad@oeaw.ac.at \\
	\addr Johann Radon Institute of Computational and Applied Mathematics (RICAM)\\
	\addr Austrian Academy of Sciences\\
	\addr Altenberger Straße 69, A-4040 Linz, Austria\\
	\AND
	\name Thomas Dittrich\thanks{ Both authors contributed equally}  \email thomas.dittrich@oeaw.ac.at \\
	\addr Johann Radon Institute of Computational and Applied Mathematics (RICAM)\\
	\addr Austrian Academy of Sciences\\
	\addr Altenberger Straße 69, A-4040 Linz, Austria
}
\editor{}

\maketitle

\begin{abstract}
    In this work, we consider the approximation capabilities of shallow neural networks in weighted Sobolev spaces for functions in the spectral Barron space.
    The existing literature already covers several cases, in which the spectral Barron space can be approximated well, i.e., without curse of dimensionality, by shallow networks and several different classes of activation function.
    The limitations of the existing results are mostly on the error measures that were considered, in which the results are restricted to Sobolev spaces over a bounded domain.
    We will here treat two cases that extend upon the existing results.
    Namely, we treat the case with bounded domain and Muckenhoupt weights and the case, where the domain is allowed to be unbounded and the weights are required to decay.
    We first present embedding results for the more general weighted Fourier-Lebesgue spaces in the weighted Sobolev spaces and then we establish asymptotic approximation rates for shallow neural networks that come without curse of dimensionality.
\end{abstract}


\section{Introduction}
\label{sec:intro}

Over the last decade and a half, deep neural networks have enabled big breakthroughs in various fields of machine learning.
The focus of the present work will be on scientific computing, where deep neural networks have contributed vastly to computational methods for solving partial differential equations with methods such as 
deep learning backwards stochastic differential equations \citep{Han18SolvingHighdimensionalPartial,E17DeepLearningBasedNumerical},
physics informed neural networks \citep{Raissi19PhysicsinformedNeuralNetworks},
deep learning variational Monte Carlo \citep{Hermann23InitioQuantumChemistry,Gerard24DeepLearningVariational},
and operator learning \citep{Anandkumar19NeuralOperatorGraph,Chen95UniversalApproximationNonlinear,Lu21LearningNonlinearOperators,Kovachki24OperatorLearningAlgorithms}.

A central question in scientific computing is, which classes of functions can be approximated well by conventional numerical methods, or more recently by neural networks \citep{Bartolucci23UnderstandingNeuralNetworks,DeVore21NeuralNetworkApproximation}.
Conventional numerical methods like finite element and finite difference approaches have been the go-to solutions for approximating solutions to partial differential equations in low dimensions.
However, as the dimensionality increases, these traditional methods face prohibitive computational costs due to the curse of dimensionality: achieving a desired accuracy $\epsilon$ requires computational resources that grow exponentially with the problem’s dimension.
The curse of dimensionality also fundamentally impacts the numerical approximation of high-dimensional functions, not just PDEs.
Standard approaches such as polynomial approximations and splines suffer from similar scaling issues, which severely restricts their applicability to high-dimensional problems.
In contrast, neural networks provide a powerful alternative to overcome the curse of dimensionality. 
Albeit not providing the best possible expressivity and approximation rates in comparison to deep neural networks \citep{Poole16ExponentialExpressivityDeep,Poggio17WhyWhenCan}, shallow neural networks are very interesting in this field from a theoretical perspective \citep{Bach17BreakingCurseDimensionality}.
This is mostly due to their 'easy' fit to existing theory such as sampling in Banach spaces \citep{Pisier80RemarquesResultatNon, Barron93UniversalApproximationBounds}, approximation with dictionaries and variation spaces \citep{Kurkova01BoundsRatesVariableBasis,Siegel23CharacterizationVariationSpaces}, and the analysis of ridge functions via the Radon transform \citep{Ongie20FunctionSpaceView, Parhi21BanachSpaceRepresenter, Unser23RidgesNeuralNetworks}.

A seminal work in this direction was the work of \citeauthor{Barron93UniversalApproximationBounds} on approximation rates for shallow neural network \citep{Barron93UniversalApproximationBounds}.
The class of functions that we consider in the present work is nowadays known as the spectral Barron space \citep{Siegel22HighOrderApproximationRates} or the Fourier-Analytic Barron space \citep{Voigtlaender22LpSamplingNumbers} and it can be seen as a polynomially weighted Fourier-Lebesgue space with integrability exponent $p=1$ \citep{Abdeljawad23SpaceTimeApproximationShallow}.
The current literature on approximation theory with shallow networks mainly focuses on cases, where the error is measured in terms of Sobolev norms with $2\leq p < \infty$ \citep{Siegel20ApproximationRatesNeural,Siegel22SharpBoundsApproximation} and Lebesgue spaces with $p=\infty$ \citep{Ma22UniformApproximationRates}.
However, an important class of error measures that can be considered for Finite Element Simulations is the class of weighted Sobolev norms (see \citep{Agnelli14PosterioriErrorEstimates,Nochetto16PiecewisePolynomialInterpolation,Heltai19ErrorEstimatesWeighted,Allendes24FiniteElementDiscretizations}).

A trivial extension of the existing approximation theory for Barron spaces towards the error being measured in weighted Sobolev spaces can be obtained for the case that the weight is bounded and the error is measured over a bounded domain.
Then, one can bound the weighted error measure by the unweighted measure using \holders{} inequality
\begin{align}
\label{eq:intro:bounded_case}
    \norm{f-f_N}{\sobolev{\ell,p}[\omega][\mathcal{U}]}
    \leq
    \norm{\omega}{L^\infty(\mathcal{U})}
    \norm{f-f_N}{\sobolev{\ell,p}[\mathcal{U}]}
\end{align}
and apply the existing theory for the unweighted case.
However, this excludes two major advantages of weighted Sobolev spaces.
Namely, unbounded weights and unbounded domains.
In the case of unbounded weights, the supremum of $\omega$ is infinite and, therefore, the right side of \cref{eq:intro:bounded_case} is infinite unless the approximation problem is trivial in the sense that $f$ can be approximated without error at a finite number of neurons.

A class of weights that is interesting for approximation on a bounded domain is the so-called class of Muckenhoupt weights.
The interest in this class initially came from the fact that the Hardy-Littlewood maximal operators, as well as a broad range of Integral operators are bounded on Muckenhoupt-weighted spaces \citep{Grafakos14ClassicalFourierAnalysis} and that they allow for significant generalizations of some Fourier inequalities such as the Hausdorff-Young inequality \citep{Heinig89FourierInequalitiesIntegral}.
From a more application-oriented perspective, Muckenhoupt-weighted spaces are of interest for finite element methods where they are used in form of weighted Sobolev spaces to study problems with singular sources \citep{Agnelli14PosterioriErrorEstimates,Nochetto16PiecewisePolynomialInterpolation,Heltai19ErrorEstimatesWeighted,Allendes24FiniteElementDiscretizations}.
Despite all this interest in weighted Sobolev spaces, to the best of our knowledge, they have not yet been studied in the context of neural networks and with a focus on the curse of dimensionality.

The case of unbounded domain is interesting insofar, as applications such as the approximation of the ground state of the electronic Schr\"{o}dinger equation requires exactly this setting \citep{Gerard24DeepLearningVariational,Dusson24PosterioriErrorEstimates} while the existing literature for neural networks solely provides universal approximation theorems \citep{Wang19ApproximationCapabilitiesNeural,vanNuland24NoncompactUniformUniversal} but no information about the approximation rate.
In this context, \citep{Wang19ApproximationCapabilitiesNeural} provides a universal approximation theorem for functions in $L^p(\rr{}\times [0,1]^d)$.
They also show that shallow networks with sigmoidal, ReLU, ELU, softplus, or LeakyReLU activation cannot universally express non-zero functions in $L^p(\rr{m}\times [0,1]^d)$ for $m>1$.
For functions that asymptotically decay to zero, \citep{vanNuland24NoncompactUniformUniversal} generalizes the universal approximation theorem to uniform convergence on $\rr{d}$.

Note that weighted error measures are not to be confused with weighted variation spaces \citep{DeVore25WeightedVariationSpaces}.
In this type of spaces, the weight is applied to the dictionary of functions with the aim of extending the space of target functions that can be approximated without curse of dimensionality.
Contrary to that, weighted error measures aim at extending the settings in which functions can be approximated well.

\subsection{Contributions and Discussion}

In this work, we will address both of the previously mentioned unbounded cases.
To do so, we will first present the necessary embedding results which show that weighted Barron spaces are embedded in weighted Sobolev spaces in order to establish that approximation of Barron functions is well defined.
Second, we will then use Maurey's sampling argument and the theory of variation spaces in order to derive approximation rates in weighted Sobolev spaces.

For bounded domain, we restrict the weight of the Sobolev space $\sobolev{\ell,p}[\omega][\set{U}]$ to be $\omega(x)=\upsilon(x)^{-\frac{1}{p^\prime}}$ with $\upsilon$ being in the Muckenhoupt class $A_{p^\prime}(\rr{d})$ with $\upsilon(x)\geq\eabs{1/\abs{x}}^{-\gamma p^\prime}$ for some $\gamma\in\rr{d}$.
Consequently, for a proper choice of $\gamma$ we can allow $\upsilon(x)=\abs{x}^{\alpha}$, which is $\upsilon\in A_{p^\prime}(\rr{d})$ for $\alpha\in(-d,d(p^\prime-1))$ (see \cref{sec:prelim}).
With that, the weight $\omega$ is allowed to have singularities for negative $\alpha$, as well as zeros for positive $\alpha$.

The full embedding result for weighted Fourier-Lebesgue spaces in weighted Sobolev spaces can be found in \cref{thm:embedding_high_degree}, we summarize this result in \cref{cor:intro:optimized_barron_embedding} (and \cref{cor:optimized_barron_embedding}) for the special case of Barron spaces.
\begin{corollary}
\label{cor:intro:optimized_barron_embedding}
    Let $d,\ell\in\nn{}$, $\gamma\in\rr{}$ with $\gamma> d/2$, $p\in[2,\infty]$, $q=2(p/2)^\prime$, $\set{U}\subset\rr{d}$ have finite volume, and
    \begin{align*}
        \omega(x)=\upsilon(x)^{-\frac{1}{p^\prime}},
    \end{align*}
    where $\upsilon$ is a radial non-decreasing function such that $\upsilon\in A_{p^\prime}(\rr{d})$ with $\upsilon(x)\geq\eabs{1/\abs{x}}^{-\gamma p^\prime}$.
    Then for any $f\in\barron{\gamma+\ell}(\set{U})$,
    \begin{align*}
        \norm{f}{\sobolev{\ell,p}[\omega][\set{U}]}
        \leq
        C_{d,\ell}
        \norm{\chF{\set{U}}}{\FL{q}[\gamma][\rr{d}]}
        \norm{f}{\barron{\gamma+\ell}[\set{U}]},
    \end{align*}
    where the constant $C_{d,\ell}$ only depends on the number of dimensions $d$ and the order $\ell$ of the Sobolev norm.
\end{corollary}
A crucial requirement for this embedding result is $\chF{\set{U}}\in\FL{q}[\gamma][\rr{d}]$.
As an example, in the unweighted one-dimensional setting with connected domain (i.e., $\gamma=0$ and $\set{U}$ being an interval) the Fourier-Transform of the characteristic function is the $\sinc$ function, which is not $L^1$-integrable.
Therefore, for $d=1$, it is essential that $q>1$ and consequently that $p<\infty$.
However, the proof of the approximation result \cref{thm:intro:approximation_sobolev} relies on Maurey's sampling argument, which is valid for Banach spaces of Rademacher-Type 2, which excludes $p=\infty$ anyway.

For unweighted and high dimensional settings, the work \citep{Ko16FourierTransformRegularity,Lebedev13FourierTransformCharacteristic} provides conditions on the Fourier-Lebesgue integrability of characteristic functions, which mostly depend on the smoothness on the boundary of the domain, with cubes being among the most well-behaved domains.
For a discussion on these works, we refer the interested reader to \cite[Section 3.1.1]{Abdeljawad23SpaceTimeApproximationShallow}.

For weighted settings, obviously, choosing a large polynomial weight (i.e., $\gamma\gg0$) prevents certain domains from being feasible.
For a short discussion of this effect, we now assume that $\set{U}$ is valid in the unweighted setting,
that is, $\chF{\set{U}}\in\FL{p_0}[0]$, and we want to find a $p_1$ such that $\chF{\set{U}}\in\FL{p_1}[\gamma_1]$ for $\gamma_1>0$.
We further assume that $\widehat{\chF{\set{U}}}$ decays (fractional) polynomially which leads to $\widehat{\chF{\set{U}}}\lesssim\eabs{\cdot}^{-\frac{d}{p_0}-\varepsilon}$ for some $\varepsilon>0$ in order for $\chF{\set{U}}\in\FL{p_0}[0]$.
We get
\begin{align*}
    \norm{\chF{\set{U}}}{\FL{p_1}[\gamma_1]}^{p_1}
    =
    \int_{\rr{d}}\left(\eabs{\xi}^{\gamma_1}\widehat{\chF{\set{U}}}(\xi)\right)^{p_1}\dee \xi
    \lesssim
    \int_{\rr{d}}\left(\eabs{\xi}^{\gamma_1-\frac{d}{p_0}-\varepsilon}\right)^{p_1}\dee \xi,
\end{align*}
which is integrable, if $p_1(\gamma_1-\frac{d}{p_0})\leq-d$.
We drop $\varepsilon$ here as it can be chosen arbitrarily small with the sole purpose of turning the inequality strict.
The conclusion is that we can at most choose $\gamma_1<\frac{d}{p_0}$, so that the left side is negative, and consequently have to choose $p_1\geq\frac{1}{\frac{1}{p_0}-\frac{\gamma_1}{d}}$.

Over an unbounded domain, the embedding is straightforward and solely requires that the weight decays sufficiently fast.
We get the following result for the Barron space as a special case of \cref{lem:embedding_unbounded}:
\begin{lemma}[Embedding of Barron Space]
    Let $d,\ell\in\nn{}$, $u\geq 0$, $p\in\rrp{}$ such that $up>d$, then
    \begin{align*}
        \norm{f}{\sobolev{\ell,p}[\eabs{\cdot}^{-u}][\rr{d}]}
        \lesssim
        \norm{f}{\barron{\ell}[\rr{d}]}
    \end{align*}
    for all $f\in\barron{\ell}[\rr{d}]$.
\end{lemma}

Subsequently, we can derive the approximation results for both cases.
First, for functions in the Barron space measured in the weighted Sobolev space over bounded domain the approximation result is stated in \cref{thm:approximation_sobolev} as follows:
\begin{theorem}[Approximation in weighted Sobolev Space]
\label{thm:intro:approximation_sobolev}
    Let $d,\ell\in\nn{}$, $\gamma\geq 0$ with $\gamma> d/2$, $p\in[2,\infty)$, $q=2(p/2)^\prime$, $\set{U}\subset\rr{d}$ such that $\chF{\set{U}}\in\FL{q}[\gamma][\rr{d}]$, and
    \begin{align*}
        \omega(x)=\upsilon(x)^{-\frac{1}{p^\prime}},
    \end{align*}
    where $\upsilon$ is a radial non-decreasing function such that $\upsilon\in A_{p^\prime}(\rr{d})$ with $\upsilon(x)\geq\eabs{1/\abs{x}}^{-\gamma p^\prime}$.
    Further, let
    $f\in \barron{\gamma+\ell+1}[\rr{d}]$ 
    and let 
    $\varrho\in \sobolev{m,\infty}[\eabs{\cdot}^s][\rr{}]$ be an activation function
    for $s>1$.
    Then,
    \begin{align*}
        \inf_{f_N\in\Sigma_{\varrho}}
        \norm{f-f_N}{\sobolev{\ell,p}[\omega][\set{U}]}
        \lesssim
        N^{-\frac{1}{2}}
        \norm{\omega}{L^p(\set{U})}
        \norm{f}{\barron{\gamma+\ell+1}[\rr{d}]}
    \end{align*}
    where the implied constant only depends on the parameters of the setting but not on the function itself.
\end{theorem}
The restrictions for this theorem come mostly from the conditions for the embedding result.
Additionally, we have to restrict to Barron functions of one order higher in order to obtain a finite variation norm in the given dictionary.

Finally, the approximation result over unbounded domain from \cref{thm:unbounded} can be stated as follows:
\begin{theorem}
    \label{thm:intro:unbounded}
    Let $d,\ell,N,m\in\nn{}$ and 
    $p,r,u,v\in\rr{}$ such that
    $2\leq p<\infty$,
    $1<r\leq v$, and
    $(u-r)p>d$.
    Furthermore, let $\varrho\in \sobolev{\ell,\infty}[\eabs{\cdot}^v][\rr{}]$ be an activation function and $\mathbb{D}_{\varrho}$ be the corresponding dictionary over $\rr{d}$.
    For every target function $f\in\barron{\ell+r}[\rr{d}]$ we get
    \begin{align*}
        \inf_{f_N\in\Sigma_{N}(\mathbb{D}_\varrho)}
        \norm{f-f_N}{\sobolev{\ell,p}[\eabs{\cdot}^{-u}][\rr{d}]}
        \lesssim
        N^{-\frac{1}{2}}
        \norm{f}{\barron{\ell+r}[\rr{d}]}.
    \end{align*}
\end{theorem}
Note, that here we require that $r>1$ and $f\in\barron{\ell+r}[\rr{d}]$.
If we were to restrict to functions $f$ with $\operatorname{supp}\{f\}\subseteq \set{U}$ for some bounded $\set{U}\subset\rr{d}$, we would obtain a special case of \cref{thm:intro:approximation_sobolev} with $\gamma=0$.
However, using directly the result for bounded domains would only require the less restrictive assumption $f\in\barron{\ell+1}[\rr{d}]$.
This difference in the asking for a strict inequality comes from the fact, that over unbounded domain, we have to ensure integrability of the polynomial bound on the activation function.

\section{Preliminaries}
\label{sec:prelim}

In this section, we cover the mathematical preliminaries, that will be necessary throughout our main contributions.

We let $\mathscr{F}$ and $\mathscr{F}^{-1}$ be the Fourier transform and the inverse Fourier transform, respectively. 
For $f\in L^1$ the pointwise definition of the Fourier transform is given by
\begin{align*}
    \F\{f\}(\xi)=\hat{f}(\xi)=\frac{1}{(2\pi)^{\frac{d}{2}}}\int_{\rr{d}}f(x)e^{-i\innerProd{x}{\xi}}\dee x
    \intertext{and for $\hat{f}\in L^1$, the pointwise definition of the inverse Fourier transform is given by}
    \F^{-1}\{f\}(x)=\frac{1}{(2\pi)^{\frac{d}{2}}}\int_{\rr{d}}\hat{f}(\xi)e^{i\innerProd{x}{\xi}}\dee \xi,
\end{align*}
where $\innerProd{\cdot}{\cdot\cdot}$ denotes the usual scalar product on $\rr{d}$.

\subsection{Weighted Function Spaces: Definition and Properties}

The central elements of our work are weighted Sobolev spaces and Fourier-Lebesgue- (respectively Barron-) spaces as classes of target functions.
We will here continue by first introducing the concept of weight functions and the special class of Muckenhoupt weights.
Secondly, we introduce weighted Sobolev spaces, weighted Fourier-Lebesgue spaces and spectral Barron spaces.
Finally, we present generalizations of the Housdorff-Young inequality and Young's convolution inequality.

\begin{definition}[{Muckenhoupt Weights \citep{Heinig89FourierInequalitiesIntegral}}]
    Let $d\in\nn{}$. A nonnegative and Lebesgue-measurable function $\omega:\rr{d}\to\rr{}$ is called a Muckenhoupt Weight of class $A_p(\rr{d})$ for $p\in(1,\infty)$, if it is locally integrable and there is a constant $C>0$ such that for all $d$-balls $\set{B}\subset\rr{d}$ with volume $\abs{\set{B}}$ it holds that
    \begin{align*}
        \frac{1}{\abs{\set{B}}}
        \left(\int_\set{B}\omega(x)\dee x\right)^\frac{1}{p}
        \left(\int_\set{B}\omega(x)^{1-p^\prime}\dee x\right)^\frac{1}{p^\prime}
        \leq C.
    \end{align*}
    If $\omega$ is solely nonnegative and Lebesgue-measurable, we refer to it as a weight function.
\end{definition}

A special example of Muckenhoupt weights of the class $A_p$ is given by $\omega: \mathbb{R}^d \to \mathbb{R}$
defined by $\omega(x)=\abs{x}^\alpha$,
where $\alpha\in(-d,d(p-1))$ 
see, e.g.,
\cite[Example 7.1.7]{ Grafakos14ClassicalFourierAnalysis}.

\begin{definition}[Weighted Function Spaces]
    Let $p\in[1,\infty]$, $d\in\nn{}$, $m\in\zzp{}$, $\set{U}\subseteq\rr{d}$, and $\omega$ be a weight function on $\rr{d}$.
    We define the weighted Lebesgue space, the weighted Sobolev space, and the weighted Fourier-Lebesgue space as
    \begin{align*}
        L^p(\omega;\set{U})
        &:=
        \{f:\set{U}\to\rr{}|\omega f\in L^p(\set{U})\},\\
        \sobolev{m,p}[\omega][\set{U}]
        &:=
        \{f:\set{U}\to\rr{}|\forall \alpha\in\zzp{d}\,\text{with}\,\abs{\alpha}_1\leq m\,\text{it holds}\,\partial^\alpha f\in L^p(\omega;\set{U})\},\\
        \intertext{and}
        \FL*{p}[\omega][\set{U}]
        &:=
        \{f:\set{U}\to\rr{}|\exists f_e\in L^1(\rr{d})\,\text{with}\,f_e|_\set{U}=f\,\text{such that}\,\hat{f}_e\in L^p(\omega;\set{U})\},
    \end{align*}
    respectively.
    These spaces are normed spaces equipped with the norms
    \begin{align*}
        \norm{f}{L^p(\omega;\set{U})}&:=\left(\int_\set{U}\abs{\omega(x)f(x)}^p\dee x\right)^{\frac{1}{p}},\\
        \norm{f}{\sobolev{m,p}[\omega][\set{U}]}&:=\left(\sum\nolimits_{\abs{\alpha}\leq m}\norm{\omega\partial^\alpha f}{L^p(\set{U})}^p\right)^{\frac{1}{p}},\\
        \intertext{and}
        \norm{f}{\FL*{p}[\omega][\set{U}]}&:=\inf_{\substack{f_e\in L^1\\\left.f_e\right|_\set{U}=f}}
        \norm{\omega\hat{f}_e}{L^{p}(\rr{d})},
    \end{align*}
    respectively.
    For $p=\infty$, we make the obvious modifications for $L^p(\rr{d})$ and replace the summation in the weighted Sobolev norm by a maximum.
\end{definition}
Note that in the case $\set{U}=\rr{d}$, the infimum in the Fourier-Lebesgue norm is over a single element (of equivalence classes), which therefore results in the implicitly requirement that $f\in L^1(\rr{d})$.
\begin{definition}[Spectral Barron Space]
    Let $d\in\nn{}$ and $\set{U}\subseteq\rr{d}$.
    In the special case of (fractional) polynomial weights of order $s\in\rr{}$ together with $p=1$, the weighted Fourier-Lebesgue space is referred to as the so-called spectral Barron space.%
    \footnote{The existing literature sometimes differentiates between the polynomially weighted spectral Barron spaces and exponentially weighted spectral Barron spaces (see, e.g.,
    \citep{Siegel22HighOrderApproximationRates,Abdeljawad24ApproximationRatesFrechet}). In the present work, however, in the present work we are only dealing with the polynomially weighted case and will therefore skip this prefix throughout the work.}
    We denote this by
    \begin{align*}
        \barron{s}[\set{U}]:=\FL*{1}[\eabs{\cdot}^s][\set{U}].
    \end{align*}
\end{definition}

\begin{remark}[Simplified Notation]
    Throughout the current work, we will consider the following simplifications in the notation:
    \begin{itemize}
        \item In case that $\set{U}=\rr{d}$, we drop the domain from the notation whenever the dimension is clear from the context, i.e., $\FL*{p}[\omega]:=\FL*{p}[\omega][\rr{d}]$.
        \item In the case of (fractional) polynomial weights of order $s\in\rr{}$, we reduce the notation for the weights in the Fourier-Lebesgue spaces such that we only denote the polynomial degree, i.e., $\FL{p}[s][\set{U}]:=\FL*{p}[\eabs{\cdot}^s][\set{U}]$.
        \item For the unweighted Fourier-Lebesgue spaces (i.e., polynomially weighted with $s=0$) we skip the weight all together, i.e., $\FL{p}[][\set{U}]:=\FL{p}[0][\set{U}]$.
    \end{itemize}
\end{remark}
A simple embedding result between Fourier-Lebesgue spaces and the Barron space can be obtained similarly as in the proof of \cite[Theorem 3.9]{Abdeljawad23SpaceTimeApproximationShallow} via Jensen's inequality.
\begin{proposition}[{Higher-Order Embedding \citep{Abdeljawad23SpaceTimeApproximationShallow}}]
Let $d\in\nn{}$, $\kappa\in\rr{}$, $t\geq 1$, and $\sigma=(d+1)(1-\frac{1}{t})$, then
\begin{align*}
    \norm{f}{\barron{\kappa}}
    &\leq \frac{1}{\norm{\eabs{\cdot}^{-(d+1)}}{L^1(\rr{d})}^{1-\frac{1}{t}}}\norm{f}{\FL{t}[\kappa+\sigma]}
\end{align*}
\end{proposition}

In our result we require generalizations of the Hausdorff-Young inequalities between weighted Fourier-Lebesgue norms and weighted Lebesgue norms.
Namely, those generalizations are \cite[Theorem 2.9 and Theorem 2.10]{Heinig89FourierInequalitiesIntegral}, which strongly rely on Muckenhoupt weights.
In order to have a self-contained work we provide these results here:
\begin{lemma}[{Generalized Hausdorff-Young Inequality: Type I \cite[Theorem 2.9]{Heinig89FourierInequalitiesIntegral}}]
    \label{lem:weighted_hausdorff_young209}
    Let $d\in\nn{}$, $1<p\leq q\leq p^\prime$, and suppose $\upsilon$ is a radial function and radially non-decreasing, such that    
    \begin{align}\label{eq:weighted_hausdorff_young209}
        \omega(x)
        =\upsilon(x)^{\frac{1}{p}}
        \quad\text{and}\quad
        \vartheta(x)
        =\abs{x}^{d\left(\frac{1}{p^\prime}-\frac{1}{q}\right)}
         \omega\left(\frac{1}{\abs{x}}\right),
    \end{align}
    then there is a constant $C>0$ such that
    \begin{align*}
        \norm{f}{\FL{q}(\vartheta)}
        \leq C
        \norm{f}{L^{p}(\omega)}
    \end{align*}
    if and only if $\upsilon\in A_{p}(\rr{d})$.
\end{lemma}

\begin{lemma}[{Generalized Hausdorff-Young Inequality: Type II \cite[Theorem 2.10]{Heinig89FourierInequalitiesIntegral}}]
    \label{lem:weighted_hausdorff_young210}
    Let $d\in\nn{}$, $1<p^\prime\leq q\leq p$, and suppose $\upsilon$ is a radial function and radially non-decreasing, such that    
    \begin{align}\label{eq:weighted_hausdorff_young210}
        \omega(x)
        =\upsilon(x)^{-\frac{1}{p^\prime}}
         \quad\text{and}\quad
        \vartheta(x)
        =\abs{x}^{d\left(\frac{1}{p^\prime}-\frac{1}{q}\right)}
         \omega\left(\frac{1}{\abs{x}}\right),
    \end{align}
    then there is a constant $C>0$ such that
    \begin{align*}
        \norm{f}{\FL{p}(\omega)}
        \leq C
        \norm{f}{L^{q}(\vartheta)}
    \end{align*}
    if and only if $\upsilon\in A_{p^\prime}(\rr{d})$.
\end{lemma}

Finally, we require the following result from \citep{Toft15SharpConvolutionMultiplication} to obtain estimates for convolution and multiplication in weighted Lebesgue and Fourier-Lebesgue spaces.

\begin{assumption}[Toft-Young Functional]\label{assump}
    Let $t_j \in \rr{}, \tau_j \in[1, \infty], j=0,1,2$, and let $R(\tau) = 2-\frac{1}{\tau_0} - \frac{1}{\tau_1}-\frac{1}{\tau_2}$.
Assume that $0 \leq R(\tau) \leq 1 / 2$, and that 
\begin{align}
&\label{eq:toftCond1} 0 \leq t_j+t_k, \quad j, k=0,1,2, \quad j \neq k, \\
&\label{eq:toftCond2} 0 \leq t_0+t_1+t_2-d R(\tau),
\end{align}
hold true with strict inequality in \cref{eq:toftCond2} when $R(\tau)>0$ and $t_j=d R(\tau)$ for some $j=0,1,2$.
\end{assumption}

\begin{proposition}[{Generalized Young Convolution Inequality \cite[Theorem 2.2(1)]{Toft15SharpConvolutionMultiplication}}]\label{prop:ToftConvThm}
Let $t_0,t_1,t_2\in\rr{}$ and $\tau_0,\tau_1,\tau_2\in[1,\infty]$ fulfill \cref{assump}.
Then the map $\left(f_1, f_2\right) \mapsto f_1 \conv{} f_2$ on $C_0^{\infty}\left(\rr{d}\right)$ extends uniquely to a continuous map from $L_{t_1}^{\tau_1}\left(\rr{d}\right) \times L_{t_2}^{\tau_2}\left(\rr{d}\right)$ to $L_{-t_0}^{q}\left(\rr{d}\right)$ with $q=\tau_0^{\prime}$.
\end{proposition}

\subsection{Smoothing by Convolution}
\label{sec:prelim:smoothing}

One key element in the proof of \cref{thm:embedding_high_degree} is to show the embedding result first for functions in the Schwartz-Space of rapidly decaying functions $\schwartz$ and then use the fact that $\schwartz$ is dense in $\barron{s}$ for every $s\in\rr{}$ in order to extend the embedding to the full spectral Barron space.
In this section we will now present the standard technique of utilizing the convolution with a mollifier function in order to represent a (possibly non-smooth) function as a limit of functions in $\schwartz$.

In this regard, the first important concepts are the concepts of Mollifiers and smoothing sequences (cf. \citep{Tartar07IntroductionSobolevSpaces}):
\begin{definition}
    A function $\phi\in\schwartz$ is called a Mollifier if $\phi(x)=0$ for $\abs{x}\geq 1$ and $\norm{\phi}{L^1}=1$.
    For $\varepsilon>0$ the sequence
    \begin{align*}
        \varrho_\varepsilon(x):=
        \frac{1}{\varepsilon^d
        }\phi
        \left(\frac{x}{\varepsilon}\right)
    \end{align*}
    is called a smoothing sequence.
\end{definition}
Note that $\norm{\varrho_\varepsilon}{L^1}=1$ for all $\varepsilon>0$.

We will apply this type of smoothing sequences to characteristic functions for the domain $\Omega\subset\rr{d}$, for which we introduce the notation
\begin{align}
\label{eq:smooth_charac}
    \chEF{\Omega}{\varepsilon}
    :=
    \chF{\Omega_\varepsilon}\conv\varrho_\varepsilon
    =
    \int_{\rr{d}}\chF{\Omega_\varepsilon}(\cdot-\tau)\varrho_\varepsilon(\tau)\dee \tau,
\end{align}
where
\begin{align*}
    \Omega_\varepsilon:=\{x\in\rr{d}|\exists y\in\Omega\,\text{such that}\,\abs{x-y}\leq\varepsilon\}
\end{align*}
is an extension of $\Omega$ by a margin of width $\varepsilon$.

In terms of the weighted Fourier-Lebesgue norm, by using \holder{}'s inequality similar to \cite[Section 2.3]{Abdeljawad23SpaceTimeApproximationShallow}, this now results in the following bound 
\begin{proposition}
\label{prop:smooth_char_Lp_bound}
Let $d\in\nn{}$, $p\in[1,\infty]$, and $\omega$ be a weight function on $\rr{d}$.
Then, for all measurable $\Omega\subseteq\rr{d}$ with finite volume and $\varepsilon>0$ it holds that
\begin{align*}
    \norm{\chEF{\Omega}{\varepsilon}}{\FL*{p}[\omega]}
    =\norm{\omega\widehat{\varrho_\varepsilon}\widehat{\chF{\Omega}}}{L^p}
    &\leq 
    \norm{\widehat{\varrho_\varepsilon}}{L^\infty}
    \norm{\omega\widehat{\chF{\Omega}}}{L^p}
    \leq 
    \norm{\varrho_\varepsilon}{L^1}
    \norm{\omega\widehat{\chF{\Omega}}}{L^p}
    =
    \norm{\chF{\Omega}}{\FL*{p}[\omega]}.
\end{align*}
\end{proposition}

Regarding convergence in $L^p$, we get the following result:
\begin{proposition}[{Convergence in $L^p$}]
\label{prop:smoothing_convergence}
    Let $d\in\nn{}$, $p\in[1,\infty)$, $\set{U}\subset\rr{d}$ be bounded, and $h\in L^\infty$.
    Then,
    \begin{align*}
        \lim_{\varepsilon\to 0}
        \norm{\chEF{\set{U}}{\varepsilon} h}{L^{p}}
        =\norm{\chF{\set{U}} h}{L^{p}}.
    \end{align*}
\end{proposition}
\begin{proof}
    The result follows immediately from \holder{}'s inequality and \cite[Lemma 3.2]{Tartar07IntroductionSobolevSpaces} as
    \begin{align*}
        \norm{(\chEF{\set{U}}{\varepsilon}-\chF{\set{U}})h}{L^p}
        \leq
        \norm{\chEF{\set{U}}{\varepsilon}-\chF{\set{U}}}{L^p}
        \norm{h}{L^\infty}
        \limto{\varepsilon\to 0}0
    \end{align*}
\end{proof}

By combining smoothing sequences and truncation sequences (cf. \citep{Tartar07IntroductionSobolevSpaces}) the following density result holds:
\begin{proposition}
    Let $d\in\nn{}$, $s\in\rr{}$, and $\set{U}\subseteq\rr{d}$ measurable.
    It holds
    \begin{align*}
        \schwartz[\rr{d}]\hookrightarrow\FL{p}[s][\set{U}]
    \end{align*}
    dense with the injection map $\iota:\schwartz[\rr{d}]\to\FL{p}[s][\set{U}]$, $\iota(f):=f|_\set{U}$.
\end{proposition}

\subsection{Maurey Approximation}

In our work, we explore the approximation properties of shallow neural networks from the prospective of nonlinear dictionary approximation.
The dictionary $\mathbb{D}$ is often assumed to be an arbitrary subset of a Banach space $\set{X}$ (see e.g., \citep{DeVore98NonlinearApproximation}).
A well known example for a dictionary with $d$-dimensional input is the set of ridge functions (see \citep{Gordon01BestApproximationRidge}), i.e., composition of non-linear univariate real valued function with a linear function 
\begin{align*}
\mathbb{D}_d^{\mathrm{ridge}} := \{h(\innerProd{w}{\cdot}): w\in \rr{d}, h\in L_{\mathrm{loc}}^1(\rr{}) \}.
\end{align*}
By restricting $h$ further to be a single (non-linear) function $\varrho\in L^1_{\mathrm{loc}}$ and translations thereof, we obtain the dictionary of shallow networks with activation function $\varrho$.
For a given activation function $\varrho:\rr{}\to\rr{}$ this dictionary is defined as
\begin{align*}
    \mathbb{D}_{\varrho,d}
    :=
    \left\{
        \varrho(\innerProd{w}{\cdot}+b)
        \middle|
        w\in \rr{d},\,b\in\rr{}
    \right\}.
\end{align*}
Throughout the remainder of this work, we will omit the subscript $d$ as the dimensionality of the input space will be clear from the context.

Based on the concept of dictionaries we can now define an $N$-term approximation class:
\begin{definition}[Approximation Class]
    Let $\mathcal{X}$ be a Banach space and $\mathbb{D}\subset\mathcal{X}$
    be a dictionary.
    The corresponding $N$-term approximation class with bounded weights is
    \begin{align*}
        \Sigma_{N,M}(\mathbb{D})
        =
        \left\{
            \sum_{n=1}^N a_n g_n
            \middle|
            a_n\in\rr{},\,g_n\in\mathbb{D},\,\text{s.t.}\,\sum_{n=1}^N\abs{a_n}\leq M
        \right\}
    \end{align*}
    and the $N$-term approximation class with unbounded weights is
    \begin{align*}
        \Sigma_{N}(\mathbb{D})=\bigcup_{M\geq0}\Sigma_{N,M}(\mathbb{D}).
    \end{align*}
\end{definition}
This indeed corresponds to the set of all shallow networks of width $N$ and activation function $\varrho$ for $\mathbb{D}_{\varrho}$.

By taking the union over all possible (finite) widths and taking the closure of that set, we obtain the closure of the convex hull of the dictionary, i.e., 
\begin{align*}
    \overline{\operatorname{conv}\{\pm\mathbb{D}\}}=\overline{\bigcup_{N\in\nn{}}\Sigma_{N,1}(\mathbb{D})}.
\end{align*}
This leads to a natural function-space for shallow neural networks of infinite width, namely, the so-called variation space.
\begin{definition}[Variation Space]
    Let $\set{X}$ be a Banach space, $\mathbb{D}\subset\set{X}$ be a dictionary and $\overline{\operatorname{conv}\{\pm\mathbb{D}\}}$
    be the closure of the convex hull of the dictionary w.r.t. $\set{X}$.
    Then, the variation space corresponding to the dictionary $\mathbb{D}$ is given by
    \begin{align*}
        \mathcal{K}_1(\mathbb{D})=\{f\in\mathcal{X}\,\text{s.t.}\,\norm{f}{\mathcal{K}_1(\mathbb{D})}<\infty\}
    \end{align*}
    with the variation norm
    \begin{align*}
        \norm{f}{\mathcal{K}_1(\mathbb{D})}
        =
        \inf\{t>0|f/t\in\overline{\operatorname{conv}\{\pm\mathbb{D}\}}\}.
    \end{align*}
\end{definition}
For functions in the variation space of a dictionary, a classical approximation result from Maurey (see \citep{Pisier80RemarquesResultatNon,Siegel22SharpBoundsApproximation}) shows that they can be approximated well in type-2 Banach spaces.
Before stating said approximation result, we recall the definition of type-2 Banach spaces; further details and extensions regarding this theory can be found e.g., in \citep{Ledoux91ProbabilityBanachSpaces,Johnson01BasicConceptsGeometry}.
\begin{definition}[Type-2 Banach Space]
 A Banach space $\mathcal{X}$ is a type-2 Banach space if there exists a positive constant $C_{\mathcal{X}}$ such that for any $N \geq 1$ and $\{f_i\}_{i=1}^N \subset \mathcal{X}$ we have
$$
\left(\mathbb{E}\left\|\sum_{i=1}^N \varepsilon_i f_i\right\|_\mathcal{X}^2\right)^{1 / 2} \leq C_{\mathcal{X}}\left(\sum_{i=1}^N\left\|f_i\right\|_\mathcal{X}^2\right)^{1 / 2}
$$
where the expectation is taken over independent Rademacher random variables $\varepsilon_1, \ldots, \varepsilon_n$, i.e.,
$$
\mathbb{P}\left(\varepsilon_i=1\right)=\mathbb{P}\left(\varepsilon_i=-1\right)=\frac{1}{2} .
$$
The constant $C_{\mathcal{X}}$ is called the type-2 constant of the space $\mathcal{X}$.
\end{definition}

Most importantly for us, it is known that the spaces $L^p(\mu)$ over a measure space $(\set{U},\mathcal{A},\mu)$ are type-2 Banach spaces.
A proof for this can be found in \cite[Section 2]{Siegel22SharpBoundsApproximation}, which can immediately be extended to weighted Sobolev spaces by replacing $L^p(\mu)$ with $\sobolev{\ell,p}[\omega][\set{U}]$.
\begin{proposition}[Weighted Sobolev Spaces are of Type-2]
    Let $\ell\in\zzp{}$, $2\leq p<\infty$, $\set{U}\subseteq\rr{d}$, and $\omega$ be a weight function, then $\sobolev{\ell,p}[\omega][\set{U}]$ is a type-2 Banach space.
\end{proposition}
\begin{proof}
    The proof is analogous to \cite[Section 2]{Siegel22HighOrderApproximationRates} by replacing $L^p(\mu)$ with $\sobolev{\ell,p}[\omega][\set{U}]$.
\end{proof}

Finally, the approximation result (in the formulation of \citep{Siegel22HighOrderApproximationRates}) for functions in the variation space of a dictionary with the error being measured in a type-2 Banach space is as follows:
\begin{proposition}[{Approximation Rate in Type-2 Banach Spaces \cite[(1.17)]{Siegel22SharpBoundsApproximation}}]\label{prop:MaureyApproximation}
Let $\mathcal{X}$ be a type- 2 Banach space and $\mathbb{D} \subset \mathcal{X}$ be a dictionary with $K_{\mathbb{D}}:=\sup _{d \in \mathbb{D}}\|d\|_{\mathcal{X}}<\infty$. Then for $f \in \mathcal{K}(\mathbb{D})$, we have
$$
\inf _{f_N \in \Sigma_{N, M_f}(\mathbb{D})}\left\|f-f_N\right\|_{\mathcal{X}} \leq 4 C_{2, \mathcal{X}} K_{\mathrm{D}}\|f\|_{\mathcal{K}(\mathbb{D})} N^{-\frac{1}{2}}
$$
with $M_f=\norm{f}{\mathcal{K}(\mathbb{D})}$.
\end{proposition} 

\section{Embedding Results for Fourier-Lebesgue Spaces}
\label{sec:embedding}

In this section we provide embedding results for Fourier-Lebesgue spaces in Sobolev spaces for the case where the domain has finite volume but the weights are allowed to have singularities.
We have to split this analysis into two main parts:
\begin{enumerate}
    \item\label{it:high} Sobolev spaces with $p\in[2,\infty)$;
    \item\label{it:low} Sobolev spaces with $p\in[1,2)$.
\end{enumerate}
The reason for this split lies in the structure of the proof.
In both cases, we need to apply some variation of the Hausdorff-Young inequality, which requires that the lower part of the inequality has an integrability exponent $\geq 2$.
For \cref{it:high} we can therefore start with the Hausdorff-Young inequality and then continue with a generalized variant of Young's convolution inequality to make the dependence on the Fourier-Lebesgue space explicit on the right side.
For \cref{it:low} we first transition to high integrability exponents by means of \holder{}'s inequality and then transition to the Fourier-Lebesgue norm by means of the Hausdorff-Young inequality.

\subsection{Embedding in Sobolev Spaces with \texorpdfstring{$p\in[2,\infty)$}{p >= 2}}

\begin{theorem}[General Embedding Result for $p\in[2,\infty)$]\label{thm:embedding_high_degree}
    Let $d,\ell \in\nn{}$,
    $\gamma, p, q\in\rr{}$ such that
 	$1<p^\prime\leq q\leq p$ 
    and $\gamma\geq -\delta:= -d\left({1}/{p^\prime}-{1}/{q}\right)$,
    $\set{U}\subset\rr{d}$ have finite volume, and
    \begin{align*}
     \omega(x) = 
     \upsilon(x)^{-\frac{1}{p^\prime}}
    \end{align*}
    where $\upsilon$ is a radial non-decreasing function
    such that $\upsilon\in A_{p^\prime}(\rr{d})$
    with $\upsilon(x)\geq \eabs{1/\abs{x}}^{-\gamma p^\prime}$.
    Furthermore, set
    $\tau_0:=q^\prime$ and
    $t_0:= -\gamma -\delta$,
    and for $j\in\{1,2\}$ let 
    $t_j\in\rr{}$ and
    $\tau_j\in[1,\infty]$ such that $t_0,t_1,t_2$ (as degree of the polynomial weights) and $\tau_0,\tau_1,\tau_2$ (as integrability exponents) fulfill \cref{assump}.
    
    Then, there exists a constant $C_{d, \ell}$ which may only depend on $d$ and $\ell$ such that
	\begin{align}
        \label{eq:embedding_high_degree}
		\norm{f}{\sobolev{\ell,p}[\omega][\set{U}]}
        \leq 
      C_{d, \ell}
        \norm{\chi_{\set{U}}}{\FL{\tau_1}[t_1]}
        \norm{f}{\FL{\tau_2}[t_2 + \ell]}
	\end{align}
    for every $f\in \FL{\tau_2}[t_2 + \ell]$.
\end{theorem}

The restriction on $\set{U}$ are more relaxed in the setting of \cref{thm:embedding_high_degree} compared to the setting of
\cref{thm:approximation_sobolev}, in the sense that it is enough for
\cref{thm:embedding_high_degree} that $\set{U}$ is a measurable set with finite volume, whereas the main \cref{thm:approximation_sobolev}
requires bounded measurable sets.

\begin{proof}
    The result is trivially true if the right-hand side of \cref{eq:embedding_high_degree} is infinite,
    which is the case if
    $\chi_{U}\notin\FL{\tau_1}[t_1]$
    and/or $f\notin \FL{\tau_2}[t_2 + \ell]$.
    In order to cover the non-trivial cases
    we start with the assumptions that
    $\chi_{U}\in\FL{\tau_1}[t_1]$
    and that $f$ is a Schwartz function, i.e.,
    $f\in \mathscr{S}(\rr{d})$.
    As a consequence of $f\in\mathscr{S}(\rr{d})$,
    we have that the right side of \cref{eq:embedding_high_degree} is finite
    since $\mathscr{S}$ is stable under Fourier transform, i.e., $\F\{f\}\in\mathscr{S}(\rr{d})$.
    
    Using \cref{prop:smoothing_convergence} in the isotropic setting
    we rewrite the weighted Sobolev norm as
    \begin{align*}
        \norm{f}{\sobolev{\ell,p}[\omega][\set{U}]}^p
        =
        \sum_{\abs{\alpha}\leq \ell}\;
        \norm{\chF{\set{U}}\cdot \omega\cdot 
        \partial^\alpha f}{L^{p}}^p
        =
        \sum_{\abs{\alpha}\leq \ell}\;
        \lim_{\varepsilon \to 0}
        \norm{
         \chEF{\set{U}}{\varepsilon}\cdot \omega\cdot 
        \partial_x^\alpha f}{L^{p}}^p.
    \end{align*}
    The fact that $U\subset\rr{d}$ is bounded
    and that $\chEF{\set{U}}{\varepsilon}$ is constructed according to
    \cref{eq:smooth_charac}, implies that
    $\chEF{\set{U}}{\varepsilon}\in L^1\cap  \FL{1}$
    as a consequence of Young's convolution inequality and H\"{o}lder's inequalities (see \citep{Tartar07IntroductionSobolevSpaces}).
    The assumption $f\in\mathscr{S}$ also implies that
    $\partial^\alpha f\in L^1\cap  \FL{1}$ for any $\alpha \in \zzp{d}$.
    As a result, with \cite[Section 13.B, Page 316]{Jones01LebesgueIntegrationEuclidean}
    we conclude that
    $$
    \mathscr{F}(\chEF{\set{U}}{\varepsilon}\partial^\alpha f)=\mathscr{F}(\chEF{\set{U}}{\varepsilon})\conv{}\mathscr{F}(\partial^\alpha f)
    $$
    and further with Young's convolution inequality
    \begin{align*}
        \norm{\mathscr{F}(\chEF{\set{U}}{\varepsilon}\partial^\alpha f)}{L^1}
        \leq\norm{\chEF{\set{U}}{\varepsilon}}{\FL{1}{}}\norm{\partial^\alpha f}{\FL{1}{}},
    \end{align*}
    for any $\alpha\in \zzp{d}$.
    Namely, $\chEF{\set{U}}{\varepsilon}\partial^\alpha f \in L^1\cap  \FL{1}{}$
    for any $\alpha\in \zzp{d}$.
    Therefore, the integral representation of the inverse Fourier transform is well defined.
    As we are dealing with real valued functions
    and use the symmetric normalization of the Fourier transform,
    we can equivalently write
    \begin{align*}
        \chEF{\set{U}}{\varepsilon} {\partial^\alpha f}
        &=
        \mathscr {F}^{-1}
        \left[
        \mathscr {F}
        \left(
        \chEF{\set{U}}{\varepsilon} {\partial^\alpha f}
        \right)
        \right]
        =
        \overline{\mathscr {F}
        \left[
        \overline{\mathscr {F}
        \left(
        \chEF{\set{U}}{\varepsilon} {\partial^\alpha f}
        \right)}
        \right]},
    \end{align*}
    where the overline represents the complex conjugate.
    This enables the equivalence
    \begin{align*}
        \norm{
        \omega
        \chEF{\set{U}}{\varepsilon}\partial^\alpha f
        }{L^{p}}
        &= 
        \norm{
        \mathscr {F}^{-1}
        \left[
        \mathscr {F}
        \left(
        \chEF{\set{U}}{\varepsilon} {\partial^\alpha f}
        \right)
        \right]
        }{L^{p}(\omega)}\\
        &= 
        \norm{
        \mathscr {F}
        \left[
        \overline{
        \mathscr {F}
        \left(
        \chEF{\set{U}}{\varepsilon} {\partial^\alpha f}
        \right)}
        \right]}
        {L^{p}(\omega)}
        =
        \norm{
        \mathscr {F}
        (h_{\alpha,\beta}^\varepsilon)
        }{L^{p}(\omega)}
    \end{align*}
    where 
    $h_{\alpha,\beta}^\varepsilon
    :=
    \overline{
    \mathscr {F}
    \left(
    \chEF{\set{U}}{\varepsilon}
    \partial^\alpha f
    \right)}$.
    Using
    \cref{lem:weighted_hausdorff_young210}
    we get
    \begin{align*}
        \norm{
        \mathscr {F}
        (h_{\alpha,\beta}^\varepsilon)
        }{L^{p}(\omega)}
        &=
        \norm{
        h_{\alpha,\beta}^\varepsilon
        }{\FL{p}[\omega]}
        \leq
        C
        \norm{
        h_{\alpha,\beta}^\varepsilon
        }{L^{q} (\vartheta)}
        =
        C
        \norm{
        {\chEF{\set{U}}{\varepsilon}{\partial^\alpha a}}
        }{\FL{q} (\vartheta)},
    \end{align*}
    for
    \begin{align*}
        \vartheta(\xi)
        &=
        \abs{\xi}^{d\left(\frac{1}{p^\prime}-\frac{1}{q}\right)}
        \upsilon\left(\frac{1}{\abs{\xi}}\right)^{-\frac{1}{p^\prime}}
    \end{align*}
    The construction of $\upsilon$, specifically the lower bound $\upsilon(x)\geq \eabs{1/\abs{x}}^{-\gamma p^\prime}$, immediately implies
    \begin{align*}
        &\vartheta(\xi) 
        \leq
        \abs{\xi}^{d\left(\frac{1}{p^\prime}-\frac{1}{q}\right)}\eabs{\xi}^\gamma
        \leq
        \eabs{\xi}^{-t_0}
    \qquad\text{where}\qquad
        t_0 =  -\gamma -d\left(\frac{1}{p^\prime}-\frac{1}{q}\right).
    \end{align*} 
    As a next step we apply the Young-type inequality
    \cref{prop:ToftConvThm}
    and remove the $\varepsilon$-dependence by using \cref{prop:smooth_char_Lp_bound}
    as follows:
    \begin{align*}
        \norm{\chEF{\set{U}}{\varepsilon}{\partial^\alpha f}}{\FL{q} (\vartheta)}
        &=
        \norm{\chEF{\set{U}}{\varepsilon}{\partial^\alpha f}}{\FL*{\tau_0^\prime}[\vartheta]}\\
        &\leq
        \norm{\chEF{\set{U}}{\varepsilon}{\partial^\alpha f}}{\FL{\tau_0^\prime}[{-t_0}]}
        \leq
        \norm{\chEF{\set{U}}{\varepsilon}}{\FL{\tau_1}[{t_1}]}
        \norm{{\partial^\alpha f}}
        {\FL{\tau_2}[{t_2}]}\\
        &\leq
        \norm{\chF{\set{U}}}{\FL{\tau_1}[{t_1}]}
        \norm{{\partial^\alpha f}}
        {\FL{\tau_2}[{t_2}]}.
    \end{align*}

    The remaining part of the proof is to remove the partial derivatives from
    $\norm{{\partial^\alpha f}}{\FL{\tau_2}[t_2]}$
    in terms of $\norm{\cdot}{\FL{\tau}[\omega]}$
    as follows
    \begin{align*}
        \norm{\partial^\alpha f}
        {\FL{\tau_2}[t_2]}
        &=
        \norm{
            \eabs{\cdot}^{t_2}
            \mathscr{F}\left(\partial^\alpha f\right)
        }{L^{\tau_2}}
        \leq
        \norm{
            \eabs{\cdot}^{t_2}
            \abs{\cdot}^{\abs{\alpha}}
             \hat{f}
        }{L^{\tau_2}}
        \\&
        \leq
        \norm{
            \eabs{\cdot}^{t_2+\ell}
            \hat{f}
        }{L^{\tau_2}}
        =
        \norm{f}{\FL{\tau_2}[t_2+\ell]}.
    \end{align*}
    
    Combining all of the above steps leads to
    \begin{align*}
        &\kern-2em\norm{f}{\sobolev{\ell,p}[\omega][\set{U}]}^p
        =
        \sum_{\abs{\alpha}\leq \ell}\;
        \norm{\chF{\set{U}}\cdot \omega\cdot 
        \partial^\alpha  f}{L^{p}}^p
        \\ &
        \leq  
        \max_{\abs{\alpha}\leq \ell}
        \norm{\chF{\set{U}}\cdot \omega\cdot \partial^\alpha  f}{L^{p}}^p
        \sum_{\abs{\alpha}\leq \ell}1\;
        \\ &
        \leq
        \max_{\abs{\alpha}\leq \ell}\;
        \lim_{\varepsilon \to 0}
        C\norm{\chEF{\set{U}}{\varepsilon}{\partial^\alpha f}}{\FL{q} (\vartheta)}
        \sum_{n=0}^\ell\binom{d+n-1}{n}
         \\&=
         \max_{\abs{\alpha}\leq \ell}\;
        \lim_{\varepsilon \to 0}
        C\norm{\chEF{\set{U}}{\varepsilon}{\partial^\alpha f}}{\FL{q} (\vartheta)}
        \frac{l + 1}{d} \binom{d + l}{d - 1}
        \\&
        \leq C_{d,\ell}
        \max_{\abs{\alpha}\leq \ell}\;
        \norm{\chF{\set{U}}}{\FL{\tau_1}[ {t_1}]}
        \norm{f}{\FL{\tau_2}[{t_2+ \ell}]}
        \\
        &=
       C_{d,\ell}\norm{\chF{\set{U}}}{\FL{\tau_1}[{t_1}]}
        \norm{f}{\FL{\tau_2}[{t_2+ \ell}]},
    \end{align*}
    where $C_{d,\ell} = C\frac{l + 1}{d} \binom{d + l}{d - 1} $.
    Finally, the fact that $\eabs{\cdot }^{t_2+ \ell}$ 
    is a polynomial weight and the density argument
    of $\mathscr{S}$ in the Fourier-Lebesgue space extend the argument from $f\in \mathscr{S}$ to $f\in\FL{\tau}[{t_2+ \ell}]$.
\end{proof}

As a special case of our result in \cref{thm:embedding_high_degree}, we can obtain the results of \cite[Lemma 2]{Siegel20ApproximationRatesNeural} as follows:
\begin{corollary}[{Barron space $\hookrightarrow \sobolev*{\ell}$}]
Let $d,\ell\in\nn{}$ and $\set{U}\subset \rr{d}$ be a domain with finite volume.
Then for any $f\in\barron{\ell}$
we have
\begin{align*}
    \norm{f}{\sobolev*{\ell}[\set{U}]}\leq C_{d,\ell}\abs{\set{U}}^{\frac{1}{2}}\norm{f}{\barron{\ell}}.
\end{align*}
\end{corollary}

\begin{proof}   
This is an immediate consequence of \cref{thm:embedding_high_degree}
with the choice $\gamma=0$,
$\omega\equiv\upsilon\equiv 1$,
$p=q=2$, $t_1=t_2=0$, $\tau_1=2$, and $\tau_2=1$.
\end{proof}

For a more general result for Barron spaces, we consider the case $\tau_2=1$ 
for which \cref{assump} simplifies to
\begin{align}
\label{eq:barron_toft_young}
    0\leq 2-\frac{1}{\tau_0}-\frac{1}{\tau_1}-1=1-\frac{1}{\tau_0}-\frac{1}{\tau_1}\leq \frac{1}{2}.
\end{align}
We observe that the requirement $1<p^\prime\leq q\leq p$ from the weighted Hausdorff-Young inequality implies that $p\in[2,\infty)$
and furthermore with $\tau_0=q^\prime$ it follows that $\tau_0\in[p^\prime,p]$ is necessary and sufficient to fulfill this condition.

It is immediately clear that $p\geq \tau_0\geq \max\{p^\prime,\tau_1^\prime\}$ (or equivalently $\min\{p,\tau_1\}\geq \tau_0^\prime\geq p^\prime$) is necessary and sufficient in order to satisfy the non-negativity in \cref{eq:barron_toft_young} and the assumption for the weighted Hausdorff-Young inequality.
We observe
\begin{itemize}
    \item $\tau_1\in [1,p^\prime)$ violates this condition;
    \item from $\tau_1\in[p^\prime,2]$ it follows that the upper bound in \cref{eq:barron_toft_young} is trivially fulfilled and that $\tau_0\in[\tau_1^\prime,p]$ is sufficient to satisfy the non-negativity in \cref{eq:barron_toft_young} and also the assumption from the weighted Hausdorff-Young inequality as $p^\prime\leq \tau_1^\prime$;
    \item for $\tau_1\in(2,p]$ we get the additional constraint that $\tau_0\leq 2(\tau_1/2)^\prime$ (i.e., the \holder{} conjugate of $\tau_1/2$). That is, we can choose $\tau_0\in[\tau_1^\prime,\min\{p,2(\tau_1/2)^\prime\}]$;
    \item for $\tau_1>p$ implies that $\tau_1^\prime<p^\prime$ and therefore we can further restrict the feasible set to $\tau_0\in[p^\prime,\min\{p,2(\tau_1/2)^\prime\}]$. Note, that this interval is never empty as $2(\tau_1/2)^\prime\geq 2$, even in the case $\tau_1=\infty$.
\end{itemize}
Formally, this is as follows:

\begin{corollary}[General Embedding of Barron Spaces]
\label{cor:general_barron_embedding}
    Let $d,\ell \in\nn{}$,
    $p\in[2,\infty)$, $\tau_1\in[p^\prime,\infty]$
    \begin{align*}
        \tau_0\in\begin{cases}
            [\tau_1^\prime,p],&\tau_1\in[p^\prime,2],\\
            [\tau_1^\prime,\min\{p,2(\tau_1/2)^\prime\}],&\tau_1\in(2,p],\\
            [p^\prime,\min\{p,2(\tau_1/2)^\prime\}],&\tau_1\in(p,\infty]
        \end{cases}
    \end{align*}
     $\set{U}\subset\rr{d}$ with finite volume, and
    \begin{align*}
        \omega(x) = 
        \upsilon(x)^{-\frac{1}{p^\prime}}
    \end{align*}
    where $\upsilon$ is a radial non-decreasing function
    such that $\upsilon\in A_{p^\prime}(\rr{d})$
    with $\upsilon(x)\geq \eabs{1/\abs{x}}^{-\gamma p^\prime}$,
    where $\gamma\geq \delta:= -d\left({1}/{p^\prime}-{1}/{\tau_0^\prime}\right)$.
    Furthermore, let
    $t_0:= -\gamma -\delta$
    and let $t_1,t_2\in\rr{}$ such that \cref{assump} is fulfilled.
 
    Then, for all $f\in\barron{t_2+\ell}[\set{U}]$
    \begin{align*}
        \norm{f}{\sobolev{\ell,p}[\omega][\set{U}]}
        \leq 
        C_{\ell,d}
        \norm{\chF{\set{U}}}{\FL{\tau_1}[t_1]}
        \norm{f}{\barron{t_2+\ell}[\set{U}]}.
    \end{align*}
\end{corollary}
In case that $p\geq 4$, we can further simplify the expression in the last two cases as the expression $(\tau_1/2)^\prime$ is monotonically decreasing in $\tau_1$.
We get $2(\tau_1/2)^\prime\leq 2(p/2)^\prime = 2p/(p-2)\leq 2p/(4-2)=p$.

Ideally for the approximation result in \cref{sec:bounded_approximation} we consider Barron functions (i.e., $\tau_2=1$) and would like to minimize the polynomial degree $t_1,t_2$ of the weights as much as possible such that the constraints on the domain are minimized and the class of functions that is embedded in $\sobolev{\ell,p}[\omega][\set{U}]$ is maximized.
To do so, we observe first that necessarily
\begin{align*}
    t_1,t_2\geq-t_0=\gamma+\delta=\gamma+d\left(\frac{1}{p^\prime}-\frac{1}{\tau_0^\prime}\right),
\end{align*}
where the right-hand side is minimized by choosing $\tau_0=p$ and the left-hand side by choosing equality.
Note that we are not necessarily free to choose $\gamma$ as this might necessarily be required to be strictly positive in order to accommodate for weights with singularities.
For the Toft-Young functional we get
$$
R(\bm{\tau})=1-\frac{1}{p}-\frac{1}{\tau_1}\in[0,1/2],
$$
which results in the necessary condition $\tau_1\in[p^\prime,2(p/2)^\prime]$.
In order to keep the restrictions on the domain as relaxed as possible, we choose
$$
\tau_1=2(p/2)^\prime=\frac{2p}{p-2}\; \text{(with $\tau_1=\infty$ for $p=2$)} 
$$
and consequently get $R(\bm{\tau})=1/2$.
As a last step, we have to make sure that either the inequality \cref{eq:toftCond2} in \cref{assump} is strict or that $\gamma > dR(\bm{\tau})=d/2$.
Due to the choice $t_1=t_2=-t_0=\gamma$, we would get equality in \cref{eq:toftCond2} if $\gamma=d/2$, thus $\gamma> d/2$ is necessary and sufficient for the other choices to be valid.

\begin{corollary}[]
\label{cor:optimized_barron_embedding}
    Let $d,\ell\in\nn{}$, $\gamma\in\rr{}$ with $\gamma > d/2$, $p\in[2,\infty]$, $q=2(p/2)^\prime$, $\set{U}\subset\rr{d}$ have finite volume, and
    \begin{align*}
        \omega(x)=\upsilon(x)^{-\frac{1}{p^\prime}},
    \end{align*}
    where $\upsilon$ is a radial non-decreasing function such that $\upsilon\in A_{p^\prime}(\rr{d})$ with $\upsilon(x)\geq\eabs{1/\abs{x}}^{-\gamma p^\prime}$.
    Then for any $f\in\barron{\gamma+\ell}(\set{U})$,
    \begin{align*}
        \norm{f}{\sobolev{\ell,p}[\omega][\set{U}]}
        \leq
        C_{d,\ell}
        \norm{\chF{\set{U}}}{\FL{q}[\gamma]}
        \norm{f}{\barron{\gamma+\ell}(\set{U})},
    \end{align*}
    where the constant $C_{d,\ell}$ only depends on the number of dimensions $d$ and the order $\ell$ of the Sobolev norm.
\end{corollary}

For more general choice for the space of target functions we can get the following result.
\begin{corollary}[Conjugate FL-Spaces]
\label{cor:conjugate_FL}
    Let $d,\ell\in\nn{}$,
    $\gamma> d/2$,
    $\set{U}\subset\rr{d}$ have finite volume,
    and
    \begin{align*}
     \omega(x) = 
     \upsilon(x)^{-\frac{1}{2}}
    \end{align*}
    where $\upsilon$ is a radial non-decreasing function
    such that $\upsilon\in A_{2}(\rr{d})$
    with $\upsilon(x)\geq \eabs{1/\abs{x}}^{-2\gamma}$.
    Then for any $f\in\FL{\tau}[\gamma+\ell][\set{U}]$
    \begin{align*}
        \norm{f}{\sobolev*{\ell}[\omega][\set{U}]}
        \leq
        C_{d,\ell}
        \norm{\chF{\set{U}}}
            {\FL{\tau^\prime}[\gamma]}
        \norm{f}{\FL{\tau}[\gamma+\ell][\set{U}]}.
    \end{align*}
\end{corollary}
\begin{proof}
The statement immediately follows from \cref{thm:embedding_high_degree} by choosing
$\tau_1^\prime=\tau_2=\tau\in[1,\infty]$, 
$t_1=t_2=-t_0=\gamma$,
and $p=q=2$.
\end{proof}

\subsection{Embedding in Sobolev Spaces with \texorpdfstring{$p\in[1,2]$}{p in [1,2]}}

\begin{lemma}[Low Degree Lemma]
\label{lem:embedding_low_degree}
    Let $d \in\nn{}$ and
    $p,q, r \in (1, \infty)$ such that  $1< q\leq r\leq q^\prime$ and $p\leq r$.
    Let
    \begin{align*}
     \vartheta(x) = \abs{x}^{d(\frac{1}{q^\prime} - \frac{1}{r})}
     \omega(x) \qquad\text{and}\qquad \omega(x)= \upsilon\left(\frac{1}{\abs{x}}\right)^{\frac{1}{q}}
    \end{align*}
    with $\upsilon$ being a radial non-decreasing weight function
    such that $\upsilon\in A_{q}(\rr{d})$
    and let $\set{U} \subset \rr{d}$ be a bounded set.
    Let $f\in \FL*{q}[\omega\eabs{\cdot}^\ell]$,
    then there exists a constant $C_{d,\ell, p}$  that depends on $d, \ell$ and $p$ such that 
	\begin{align*}
		\norm{f}{\sobolev{\ell, p}[\vartheta][\set{U}]}
        \leq 
        C_{d,\ell, p}  
       \abs{\set{U}}^{\frac{1}{p}-\frac{1}{r}}
       \norm{f}{\FL*{q}[\omega\eabs{\cdot}^\ell][\rr{d}]}
	\end{align*}
 for any $\ell\in \zzp{}$.
\end{lemma}

\begin{proof}
    First, we assume that $f\in \mathscr{S}(\rr{d})$
    and recall that
    \begin{align*}
        \norm{f}{\sobolev{\ell, p}[\vartheta][\set{U}]}^p
        &:= \sum_{\abs{\alpha}\leq \ell}
        \int_{ \set{U}}
        \abs{\vartheta(x) \partial_x^\alpha f(x)}^{p}
        \dee x
        \\  &\leq 
        \max_{\abs{\alpha}\leq \ell}\;
        \int_{ \set{U}}
        \abs{\vartheta(x) \partial_x^\alpha f(x)}^{p}
        \dee x
        \sum_{\abs{\alpha}\leq \ell}1
        \\  &\leq 
        \frac{\ell+1}{d}\binom{d+\ell}{d-1}
        \max_{\abs{\alpha}\leq \ell}\;
        \int_{ \set{U}}
        \abs{\vartheta(x) \partial_x^\alpha f(x)}^{p}
        \dee x
        \\&
        = \frac{\ell+1}{d}\binom{d+\ell}{d-1} 
        \max_{\abs{\alpha}\leq \ell}\;
        \norm{\chF{\set{U}} \vartheta 
        \partial^\alpha f}{L^{p}}^p,
    \end{align*}
    for any $\ell\in\zzp{}$ and multiindices $\alpha\in \zzp{d}$.
    
    We now assume $\alpha\in\zzp{d}$ to be fixed with $\abs{\alpha}\leq \ell$.
    By \holder{}'s inequality
    with \(\frac{1}{p} = \frac{1}{t} + \frac{1}{r}\)
    we get
    \begin{align*}
        \norm{\chF{\set{U}}\vartheta\partial^\alpha f}{L^{p}}
        \leq
        \abs{\set{U}}^{\frac{1}{t}}
        \norm{\vartheta\partial^\alpha f}{L^{r}}.
    \end{align*}
    Since $f\in \mathscr{S}$ it follows that
    $\partial^\alpha f$ belongs also to $\mathscr{S}$.
    Hence $\partial^\alpha f$ and $ \F\left(\partial^\alpha f\right)$
    belong to $L^1$.
    Therefore, we have pointwise equality in
    \begin{equation*}
    \partial^\alpha f = \F{}^{-1}
        \left(
            \F{}\left(\partial^\alpha f\right)
        \right).
    \end{equation*}
    Furthermore, $\vartheta$ is a weight
    of the form given in \cref{eq:weighted_hausdorff_young209} and therefore
    by \cref{lem:weighted_hausdorff_young209}
    there is a constant $C>0$ such that
    \begin{align*}
    \norm{\partial^\alpha f}{L^{r}(\vartheta)}
    =
        \norm{ \F{}\left(\partial^\alpha f\right)}{\FL{r}(\vartheta)}
        \leq C
        \norm{ \F{}\left(\partial^\alpha f\right)}{L^{q}(\omega)},
    \end{align*}
    where $1<q\leq r\leq q\prime$
    and $\omega(x)= v({1}/{\abs{x}})^{\frac{1}{q}}$.
    Consequently, 
    we conclude that
    \begin{align*}
        \norm{ \F{}\left(\partial^\alpha f\right)}{L^{q}(\omega)}
        &\leq
            \norm{\omega \abs{\cdot}^{\abs{\alpha}}\F{}(f)}
                 {L^{q}}.
    \end{align*}

    All together, we get
    \begin{align*}
        \norm{f}{\sobolev{\ell, p}[\vartheta][\set{U}]}
           &\leq  \left( \frac{\ell+1}{d}\binom{d+\ell}{d-1}  \right)^\frac{1}{p} \abs{\set{U}}^{\frac{1}{t}}
          \norm{\partial^\alpha f}{L^{r}(\vartheta)}
           \leq C_{d,\ell ,p} \abs{\set{U}}^{\frac{1}{t}}  \norm{ \F{}\left(\partial^\alpha f\right)}{L^{q}(\omega)}
            \\
        &\leq  C_{d,\ell ,p}  \abs{\set{U}}^{\frac{1}{t}}  \norm{\omega \abs{\cdot}^{\abs{\alpha}}\F{}(f)}
                     {L^{q}}
           \leq C_{d,\ell ,p} 
           \abs{\set{U}}^{\frac{1}{t}}
           \norm{\omega
           \eabs{\cdot}^{\abs{\alpha}}
           \F{}(f)}{L^{q}}
        \\
            &\leq C_{d,\ell ,p} 
           \abs{\set{U}}^{\frac{1}{t}}
           \norm{f}{\FL*{q}[\omega \eabs{\cdot}^\ell]},   
    \end{align*}    
    for any $\ell\in \zzp{}$, where $C_{d,\ell ,p} = C \left( \frac{\ell+1}{d}\binom{d+\ell}{d-1}  \right)^\frac{1}{p}$.
    This concludes the proof of the lemma
    by using the fact that
    $\mathscr{S}$ is dense in
    $\FL*{q}[\omega \eabs{\cdot}^\ell]$,
    for any $\ell\in \zzp{}$ and $q\in [1,\infty]$.
\end{proof}

\section{Approximation of Fourier-Lebesgue Spaces}
\label{sec:approximation}

As a second part of our main contributions, we now deal with function approximation.
For doing so, we make use of Maurey's sampling argument \citep{Pisier80RemarquesResultatNon} and follow a similar approach as in \citep{Barron93UniversalApproximationBounds,Siegel20ApproximationRatesNeural,Abdeljawad23SpaceTimeApproximationShallow}.
Since the techniques used for bounded and unbounded domains are different, we split this section into two parts. First, we address the approximation of functions with error in bounded domains, and next we analyze the unbounded case, as seen in \cref{thm:approximation_sobolev} and \cref{thm:unbounded}, respectively.

\subsection{Error Measure over Bounded Domain}
\label{sec:bounded_approximation}

In this section, we aim to approximate functions in a given weighted Barron space using shallow neural networks, where the error measure is  the weighted Sobolev norm on a bounded domain.
It is worth mentioning that our findings generalize the existing literature in the sense that we allow the weights to exhibit singularities and permit arbitrary integrability exponents such that the space is of Rademacher type 2.

\begin{theorem}[Approximation in weighted Sobolev Space]
\label{thm:approximation_sobolev}
    Let $d,\ell\in\nn{}$, $\gamma\geq 0$ with $\gamma\neq d/2$, $p\in[2,\infty)$, $q=2(p/2)^\prime$, $\set{U}\subset\rr{d}$ such that $\chF{\set{U}}\in\FL{q}[\gamma]$, and
    \begin{align*}
        \omega(x)=\upsilon(x)^{-\frac{1}{p^\prime}},
    \end{align*}
    where $\upsilon$ is a radial non-decreasing function such that $\upsilon\in A_{p^\prime}(\rr{d})$ with $\upsilon(x)\geq\eabs{1/\abs{x}}^{-\gamma p^\prime}$.
    Further, let
    $f\in \barron{\gamma+\ell+1}$ 
    and let 
    $\varrho\in \sobolev{m,\infty}[\eabs{\cdot}^s][\rr{}]$ be an activation function
    for $s>1$.
    Then,
    \begin{align*}
        \inf_{f_N\in\Sigma_{\varrho}}
        \norm{f-f_N}{\sobolev{\ell,p}[\omega][\set{U}]}
        \lesssim
        N^{-\frac{1}{2}}
        \norm{\omega}{L^p(\set{U})}
        \norm{f}{\barron{\gamma+\ell+1}}
    \end{align*}
    where the implied constant only depends on the parameters of the setting but not on the function itself.
\end{theorem}

\begin{proof}
    In this proof we take a similar approach to \citep{Barron93UniversalApproximationBounds,Siegel20ApproximationRatesNeural,Abdeljawad23SpaceTimeApproximationShallow}.
    That is, we first show that the target function can be represented as a infinite convex combination of elements of some dictionary
    and second, we use Maurey's sampling argument (see \citep{Pisier80RemarquesResultatNon,Barron93UniversalApproximationBounds}) to provide the approximation rate.

    The first step in this approach is to express the Fourier basis in terms of the target function.
    The approach to do so is to start with a linear shift in the Fourier transform of the activation function
    \begin{align*}
        \hat{\varrho}(\tau)
        &=\frac{1}{\sqrt{2\pi}}\int_{\rr{}}\varrho(t)e^{-i\tau t}dt
        =\frac{1}{\sqrt{2\pi}}\int_{\rr{}}\varrho(\innerProd{\xi}{x}+b)e^{-i\tau (\innerProd{\xi}{x}+b)}\dee b
    \end{align*}
    which allows us to express the exponential term as
    \begin{align*}
        e^{i\tau\innerProd{\xi}{x}}
        =\frac{1}{\sqrt{2\pi}\hat{\varrho}(\tau)}\int_{\rr{}}\varrho(\innerProd{\xi}{x}+b)e^{-i\tau b}\dee b
    \end{align*}
    under the assumption that $\tau\neq 0$ and $\hat{\varrho}(\tau)\neq 0$.
    Inserting this into the Fourier transform of the target function leads to the representation    
    \begin{align*}
        f(x)=\frac{1}{(2\pi)^\frac{d}{2}}\int_{\rr{d}}e^{i\innerProd{\xi}{x}}\hat{f}(\xi)\dee \xi
        =\frac{1}{(2\pi)^\frac{d+1}{2}\hat{\varrho}(\tau)}
        \int_{\rr{d}}\int_{\rr{}}
        \varrho\left(\frac{\innerProd{\xi}{x}}{\tau}+b\right)\hat{f}(\xi) e^{-i\tau b}\dee b \dee \xi.
    \end{align*}
    In the next step we split this representation into the elements of some dictionary and the measure that represents our function in this dictionary.
    To do so, we introduce the modified weight
    \begin{align*}
        \widetilde\varphi(\xi,b)=(1+(\abs{b}-R_{\set{U}}\abs{\xi/\tau})_+)^s
    \qquad\text{with}\qquad
        R_{\set{U}}=\sup_{x\in \set{U}}\abs{x}
    \end{align*}
    and extend the representation of $f$ as follows
    \begin{align*}
        f(x)
        =C_{\varrho,d}
        \int_{\rr{d}}\int_{\rr{}}
        \frac{\widetilde\varphi(\xi,b)}{\eabs{\xi}^{\gamma + \ell}}
        \varrho\left(\frac{\innerProd{\xi}{x}}{\tau}+b\right)
        \frac{\eabs{\xi}^{\gamma + \ell}}{\widetilde\varphi(\xi,b)}
        \hat{f}(\xi)
        e^{-i\tau b}
        \dee b \dee\xi
    \end{align*}
    where the constant is
    \begin{align*}
        C_{\varrho,d}=\left((2\pi)^\frac{d+1}{2}\hat{\varrho}(\tau)\right)^{-1}.
    \end{align*}
    In this representation, the elements of our dictionary are then
    \begin{align*}
        \tilde{\varrho}(x;\xi,b)=\frac{\widetilde\varphi(\xi,b)}{\eabs{\xi}^{\gamma + \ell}}
        \varrho\left(\frac{\innerProd{\xi}{x}}{\tau}+b\right)
    \end{align*}
    with parameters $\xi$ and $b$
    and the (complex) measure, associated with the function $f$, is
    \begin{align*}
        \dee{}\mu_f(\xi,b)
        =C_{\varrho,d}
        \frac{\eabs{\xi}^{\gamma + \ell}}{\widetilde\varphi(\xi,b)}
        \hat{f}(\xi)
        e^{-i\tau b}\dee{}(\xi,b).
    \end{align*}

    Based on this measure, we can calculate the variation norm of $f$ as
    \begin{align*}
        \norm{f}{\mathcal{K}(\Sigma_{\tilde{\varrho}})}
        =\int_{\rr{d}\times\rr{}}
        \dee \abs{\mu_f}(\xi,b)
        =\norm{\mu_f(\cdot,\cdot\cdot)}{L^1(\rr{d}\times\rr{})},
    \end{align*}
    for which we now first calculate the integral over $b$.
    \begin{align*}
        I(\xi)
        &=\int_{\rr{}}\frac{1}{\tilde{\varphi}(\xi,b)}\dee b
        =2\int_{0}^\infty\frac{1}{(1+(b-R_{\set{U}}\abs{\xi/\tau})_+)^s}\dee b
        \\&
           = 2\left(\abs{\frac{R_{\set{U}}\xi}{\tau}}+\int_{0}^\infty\frac{1}{\eabs{b}^s}\dee b\right)
        \leq
        C_{\set{U},\tau,s}\eabs{\xi}
        .
    \end{align*}
    The variation norm is then given by
    \begin{align}
    \label{eq:var_norm_bound}
        \norm{f}{\mathcal{K}(\Sigma_{\tilde{\varrho}})}
        &=C_{\varrho,d}
        \int_{\rr{d}}\int_{\rr{}}
        \abs{\frac{\eabs{\xi}^{\gamma + \ell}}{\widetilde\varphi(\xi,b)}
               \hat{f}(\xi) e^{-i\tau b}}
               \dee b\dee \xi
        =C_{\varrho,d}
            \int_{\rr{d}}
            I(\xi)
            \eabs{\xi}^{\gamma + \ell}
            \abs{\hat{f}(\xi)}
            \dee\xi.
    \end{align}
    For the upper bound on the dictionary we first consider the partial derivatives for fixed $\alpha$
    \begin{align*}
        \norm{\partial^\alpha \tilde{\varrho}}{L^p(\omega;\set{U})}
        &=\norm{\omega \partial^\alpha \tilde{\varrho}}{L^p(\set{U})}
        =\frac{\widetilde{\varphi}(\xi,b)}{\eabs{\xi}^{\gamma + \ell}}
            \frac{\abs{\xi^\alpha}}{\abs{\tau}^{\abs{\alpha}}}
            \norm{\omega(\cdot)\varrho^{(\abs{\alpha})}\left(\frac{\xi}{\tau}\,\cdot\,+b\right)}{L^p(\set{U})}\\
        &\leq C_{\varrho,\varphi}
            \widetilde{\varphi}(\xi,b)
            \abs{\tau}^{-\abs{\alpha}}
            \norm{\frac{\omega(\cdot)}{\widetilde{\varphi}(\xi,b)}}{L^p(\set{U})}
        = C_{\varrho,\varphi}
            \abs{\tau}^{-\abs{\alpha}}
            \norm{\omega}{L^p(\set{U})}
     \end{align*}
    where
    \begin{align*}
        C_{\varrho,\varphi}=\norm{\varrho}{\sobolev{\ell,\infty}[\eabs{\cdot}^s][\rr{}]}.
    \end{align*}
    The final bound on the Sobolev-Norm is given by
    \begin{align}
    \label{eq:dict_bound}
        \norm{\tilde\varrho}{\sobolev{n,p}[\omega][\set{U}]}
        &=\sum_{\abs{\alpha}\leq \ell} \norm{\partial^\alpha \tilde{\varrho}}{L^p(\omega;\set{U})}\nonumber
        \\&
        \leq
        C_{\varrho,\varphi}
        \norm{\omega}{L^p(\set{U})}
        \sum_{\abs{\alpha}\leq \ell} 
            \abs{\tau}^{-\abs{\alpha}}
        \\&
        =
        C_{\varrho,\varphi}C_{\tau,\ell}
        \norm{\omega}{L^p(\set{U})}.\nonumber
    \end{align}

    With 
    the assumptions on $f$ and $\set{U}$, namely
    \begin{align*}
         f\in\barron{\gamma+\ell+1}%
        \qquad\text{and}\qquad
        \chF{\set{U}}\in\FL{q}[\gamma],
    \end{align*}
    and \cref{cor:optimized_barron_embedding}, knowing that the weighted Fourier Lebesgue spaces decreasing when the weight increases, we also get $f\in \sobolev{\ell,p}[\omega][\set{U}]$.
    Therefore, we can apply Maurey's approximation
    (cf. \cref{prop:MaureyApproximation}) to get the result with 
    $\set{X}=\sobolev{\ell,p}[\omega][\set{U}]$,
    $\norm{f}{\mathcal{K}(\Sigma_{\tilde{\varrho}})}$ given by \cref{eq:var_norm_bound},
    $K_\mathbb{D}$ given by \cref{eq:dict_bound},
    and $M=\norm{f}{\mathcal{K}(\Sigma_{\tilde{\varrho}})}$.
\end{proof}

\begin{remark}[The theory covers unbounded weights]
    Our initial claim (see \cref{sec:intro}) was that our theory is capable of treating unbounded weights.
    To see this, we take a more detailed look at the assumption $\upsilon(x)\geq\eabs{1/\abs{x}}^{-\gamma p^\prime}$.
    This assumption allows us to choose the weight $\upsilon=\abs{x}^\gamma$ and for $-d<\gamma<(p-1)d$ we get $\upsilon\in A_p(\rr{d})$.
    Thus, we can use $\sobolev{\ell,p}[\omega][\mathcal{U}]$ with $\omega(x)=\abs{x}^{-\frac{\gamma}{p^\prime}}$ weight in the error norm, which indeed has a singularity at $x=0$ for positive $\gamma$.
\end{remark}

\subsection{Error Measure over Unbounded Domain}
\label{sec:unbounded}

In this final section we focus on the case that the error in the Sobolev norm is measured over an unbounded domain.
In that regard, we have to consider a weighted Sobolev norm with decaying weight.
As a first result, we extend known embedding results for the spectral Barron space in the Sobolev space $\sobolev*{\ell}[\mathcal{U}]$ with bounded $\mathcal{U}\subset\rr{d}$ to an embedding in $\sobolev{\ell,p}[\eabs{\cdot}^{-u}]$.
\begin{lemma}[Embedding of Fourier-Lebesgue Spaces]
\label{lem:embedding_unbounded}
    Let $d,\ell\in\nn{}$, $u\geq 0$, $1\leq p^\prime\leq q\leq 2 \leq p$ such that $\frac{1}{p}=\frac{1}{r}+\frac{1}{q^\prime}$ (we make the adaptation $r=p$ if $q=1$) and $ur>d$, then
    \begin{align*}
        \norm{f}{\sobolev{\ell,p}[\eabs{\cdot}^{-u}]}
        \lesssim
        \norm{f}{\FL{q}[\ell]}
    \end{align*}
    for all $f\in\barron{\ell}$.
\end{lemma}

\begin{proof}
    Let $f\in\mathscr{S}$, then we get with $\frac{1}{p}=\frac{1}{r}+\frac{1}{q^\prime}$ via \holder{}'s inequality and the Hausdorff-Young inequality that
    \begin{align*}
        \norm{\eabs{\cdot}^{-u} \partial^\alpha f}{L^{p}}
        &\leq
        \norm{\eabs{\cdot}^{-u}}{L^{r}}\norm{\partial^\alpha f}{L^{q^\prime}}
        \leq
        \norm{\eabs{\cdot}^{-u}}{L^{r}}
        \norm{\widehat{\partial^\alpha f}}{L^{q}}\\
        &\leq
        \norm{\eabs{\cdot}^{-u}}{L^r}
        \norm{\eabs{\cdot}^\ell\widehat{f}}{L^{q}}
        =
        \norm{\eabs{\cdot}^{-u} }{L^{r}}
        \norm{f}{\FL{q}[\ell]}.
    \end{align*}
    With $ur>d$, the constant $\norm{\eabs{\cdot}^{-u} }{L^{r}}$ is finite and due to the density of $\mathscr{S}$ in $\FL{q}(\eabs{\cdot}^\ell)$ we can extend this bound to all $f\in\FL{q}(\eabs{\cdot}^\ell)$.
\end{proof}

As a step towards our approximation result for shallow neural networks, we next provide bounds on all possible neurons in the network which will later be used to develop a specific dictionary that can be uniformly bounded in $\sobolev{\ell,p}[\eabs{\cdot}^{-u}]$.
\begin{lemma}[Bound on Neurons]
\label{lem:unbounded:activation}
    Let $d,\ell,m\in\nn{}$, $2\leq p<\infty$, and $p,u,v\in\rr{}$.
    Furthermore, let $\varrho\in \sobolev{\ell,\infty}[\eabs{\cdot}^v][\rr{}]$.
    Then, for all $(\xi,b)\in\rr{d}\times\rr{}$,
    \begin{align*}
        g(\cdot;\xi,b):\rr{d}\to\rr{},\qquad \text{with}\qquad g(x;\xi,b):=\varrho(\innerProd{\xi}{x}/\tau+b)
    \end{align*}
    is in $\sobolev{\ell,p}[\eabs{\cdot}^{-u}]$ and we get
    \begin{enumerate}
        \item for $up>d$:
        \begin{align*}
            \norm{g(\cdot;\xi,b)}{\sobolev{\ell,p}[\eabs{\cdot}^{-u}]}
            \lesssim
            \abs{\xi}^\ell
            \norm{
                \eabs{\cdot}^{\frac{d-1}{p}-u}
                \eabs{\abs{\xi}/\tau \cdot +b}^{-v}
            }{L^p(\rr{})};
        \end{align*}
        \item for $1<r\leq v$ and $(u-r)p>d$:
        \begin{align*}
            \norm{g(\cdot;\xi,b)}{\sobolev{\ell,p}[\eabs{\cdot}^{-u}]}
            \lesssim
            \abs{\xi}^{\ell}
            \eabs{\min\{1,\abs{\tau}/\abs{\xi}\}\abs{b}}^{-r}.
        \end{align*}
    \end{enumerate}
\end{lemma}

\begin{proof}
In $(i)$ we start off by calculating the $L^p$ norm for fixed order $\alpha\in\zz{d}_+$ of partial derivatives ($\abs{\alpha}\leq\ell$) and fixed parameters $(\xi,b)\in\rr{d}\times\rr{}$
\begin{subequations}
\label{eq:unbounded:high_dim_integral}
\begin{align*}
	\norm{
		\partial^\alpha
		g(\cdot;\xi,b)
	}{L^{p}(\eabs{\cdot}^{-u};\rr{d})}
    &=
	\norm{
		\partial^\alpha
		\varrho(\innerProd{\xi}{\cdot}/\tau+b)
	}{L^{p}(\eabs{\cdot}^{-u};\rr{d})}
    \\&
	=
	\abs{\xi^\alpha}
	\abs{\tau}^{-\abs{\alpha}}
	\norm{
		\varrho^{(\abs{\alpha})}(\innerProd{\xi}{\cdot}/\tau+b)
	}{L^{p}(\eabs{\cdot}^{-u};\rr{d})}
	\\&
	\leq
	\abs{\xi}^\ell
	\abs{\tau}^{-\abs{\alpha}}
	\norm{
		\varrho^{(\abs{\alpha})}(\innerProd{\xi}{\cdot}/\tau+b)
	}{L^{p}(\eabs{\cdot}^{-u};\rr{d})}
	\\&
	=
	\abs{\xi}^\ell
	\abs{\tau}^{-\abs{\alpha}}
    \norm{
        \eabs{\cdot}^{-u}
        \varrho^{(\abs{\alpha})}(\innerProd{\xi}{\cdot}/\tau+b)
	}{L^p}
    \\&
    \lesssim
	\abs{\xi}^\ell
	\left(
	\int_{\rr{d}}
        \eabs{x}^{-up}
        \eabs{\innerProd{\xi}{x}/\tau+b}^{-vp}
		\dee x
	\right)^\frac{1}{p}.
\end{align*}
\end{subequations}
The implied constant is $\abs{\tau}^{-\abs{\alpha}}\norm{\varrho}{\sobolev{\ell,\infty}[\eabs{\cdot}^v][\rr{}]}$.
For $\xi=0$ we can combine all partial derivatives with $\abs{\alpha}\leq \ell$ to get
\begin{align*}
	\norm{
		\partial^\alpha
		\varrho(\innerProd{\xi}{\cdot}/\tau+b)
	}{L^{p}(\eabs{\cdot}^{-u};\rr{d})}
    \lesssim
    \eabs{b}^{-v}
    \norm{\eabs{\cdot}^{-u}}{L^p}
    \lesssim
    \abs{0}^\ell
    \norm{
        \eabs{\cdot}^{\frac{d-1}{p}-u}
        \eabs{0 + b}^{-v}
    }{L^p(\rr{})},
\end{align*}
with the implied constant depending on $\tau$, $\ell$, and $d$.

The case $\xi\neq 0$ and $d=1$ is trivial, as \cref{eq:unbounded:high_dim_integral} is the same expression as the statement of the lemma.

For the case $\xi\neq 0$ and $d>1$ we split the integral over $\rr{d}$ into the one-dimensional integral that is parallel to $\xi$ and the remaining parts that are orthogonal to $\xi$.
To do so, we denote by $V_\xi^\perp$ the basis for the $d-1$ dimensional subspace that is orthogonal to $\xi$. Then for every $x\in\rr{d}$, there is $y\in\rr{d-1}$ and $t\in\rr{}$ such that $x=V_\xi^\perp y+\abs{\xi}^{-1}\xi t$.
Overall, this transformation is unitary (i.e., no rescaling due to Jacobian determinant) and therefore
\begin{align}
\label{eq:unbounded:one_dim_integral}
	\int_{\rr{d}}
        \eabs{x}^{-up}
        \eabs{\innerProd{\xi}{x}/\tau+b}^{-vp}
		\dee x
	&=
	\int_{\rr{}}
	\int_{\rr{d-1}}
        \eabs{V_\xi^\perp y
				+
				\abs{\xi}^{-1}\xi t}^{-up}
        \eabs{\abs{\xi}t/\tau+b}^{-vp}
    \dee y
    \dee t.
\end{align}
We will simplify this argument by calculating the $(d-1)$-dimensional integral over $y$.
To do so, we first use the the equivalence of finite-dimensional $p$-norms to bound
$$
\eabs{x}=\norm{(x,1)}{1}\leq \sqrt{d+1}\norm{(x,1)}{2}=\sqrt{d+1}(1+\abs{x}^2)^\frac{1}{2}.
$$
Second, we use the fact that the basis formed by $V_\xi^\perp$ and $\xi$ are orthogonal.
Third, we perform a transformation to polar coordinates.
And fourth, we substitute $z=\sqrt{1+t^2}\tan(\theta)$.
\begin{subequations}
\label{eq:unbounded:dim_reduction}
\begin{align*}
    \kern2em&\kern-2em
    \int_{\rr{d-1}}
    \left(
        1
        +
        \abs{
            V_\xi^\perp y
            +
            \abs{\xi}^{-1}\xi t
            }^2
        \right)^{-\frac{up}{2}}
    \dee y
    \\&
    =
    \int_{\rr{d-1}}
    \left(
        1
        +
        t^2
        +
        \abs{y}^2
        \right)^{-\frac{up}{2}}
    \dee y
    \\&
    =\frac{2\pi^\frac{d-1}{2}}{\Gamma(\frac{d-1}{2})}
    \int_0^\infty z^{d-2}
    \left(
        1
        +
        t^2
        +
        z^2
        \right)^{-\frac{up}{2}}
    \dee y
    \\&
    =\frac{2\pi^\frac{d-1}{2}}{\Gamma(\frac{d-1}{2})}
    \int_0^\frac{\pi}{2}
    (1+t^2)^\frac{d-2}{2}
    \tan^{d-2}(\theta)
    (
    1
    +
    t^2
    )^{-\frac{up}{2}}
    \left(
        1
        +
        \tan^2(\theta)
        \right)^{-\frac{up}{2}}
        \frac{\sqrt{1+t^2}}{\cos^2(\theta)}
    \dee \theta
    \\&
    =\frac{2\pi^\frac{d-1}{2}}{\Gamma(\frac{d-1}{2})}
    (1+t^2)^{\frac{d-1-up}{2}}
    \int_0^\frac{\pi}{2}
    \sin^{d-2}(\theta)
    \cos(\theta)^{up-d}
    \dee \theta
    \\&
    \lesssim
    \eabs{t}^{d-1-up}.
\end{align*}
\end{subequations}
For the last bound, we used the assumption $up>(u-r)p> d>1$, which renders the exponents in the integral positive and, therefore, allows us to bound the integral by $\frac{\pi}{2}$.
The implied constant then depends solely on $d$.
Inserting \cref{eq:unbounded:one_dim_integral} and \cref{eq:unbounded:dim_reduction} into \cref{eq:unbounded:high_dim_integral} results in the following scalar integral:
\begin{align*}
	\norm{
		\partial^\alpha
		g(\cdot;\xi,b)
	}{L^{p}(\eabs{\cdot}^{-u};\rr{d})}
    &\lesssim
	\abs{\xi}^\ell
    \left(
        \int_{\rr{}}
        \eabs{t}^{d-1-up}
        \eabs{\abs{\xi}t/\tau +b}^{-vp}
        \dee t
    \right)^\frac{1}{p}
    \\&
    =
	\abs{\xi}^\ell
    \norm{
        \eabs{\cdot}^{\frac{d-1}{p}-u}
        \eabs{\abs{\xi}/\tau \cdot +b}^{-v}
    }{L^p(\rr{})}.
\end{align*}
The same asymptotic bound holds true for the Sobolev norm $\sobolev{\ell,p}[\eabs{\cdot}^{-u}]$ with an additional implied constant counting the number of partial derivatives up to order $\ell$.

For $(ii)$, we extend the bound from $(i)$.
Observe, that the value of the norm is independent of the sign of $b$ and $\tau$ and therefore, we will limit our analysis to the case where $bt<0$.
With $1<r\leq v$ we have
\begin{align*}
    \eabs{t}^{\frac{d-1}{p}-u}\eabs{\abs{\xi}t/\tau+b}^{-v}
    &\leq
    \eabs{t}^{\frac{d-1}{p}-u}\eabs{\abs{\xi}t/\tau+b}^{-r}
    \\&
    =
    \eabs{t}^{-(u-r-\frac{d-1}{p})}
    \left(
    \frac
        {1}
        {\eabs{t}\eabs{\abs{\xi}t/\tau+b}}
    \right)^{r}
\end{align*}
and define $h(t):=\eabs{t}\eabs{\abs{\xi}t/\tau+b}$.
In order to find an upper bound on the reciprocal of $h$, we instead find a lower bound on $h$.
Splitting into the three intervals $(-\infty,0)$, $[0,-\tau b/\abs{\xi})$, and $[-\tau \abs{b}/\abs{\xi},\infty)$, we see that $h$ is monotonically decaying on the first interval, concave on the second interval, and monotonically increasing on the third interval.
Thus, its minimum is obtained either for $t=0$ or $t=-\tau b/\abs{\xi}$.
That is
\begin{align*}
    \min_{t\in\rr{}} h(t)
    =
    \min\{h(0),h(-\tau b/\abs{\xi})\}
    =
    \min\{\eabs{b},\eabs{\tau b/\abs{\xi}}\}
    =1+\min\{1,\abs{\tau}/\abs{\xi}\}\abs{b}.
\end{align*}
With the assumption $(u-r)p>d$ this implies
\begin{align*}
    \norm{\eabs{t}^{\frac{d-1}{p}-u}\eabs{\abs{\xi}t/\tau+b}^{-v}}{L^p(\rr{})}
    &
    \leq
    \norm{{1}/{h^r}}{L^\infty}
    \norm{\eabs{t}^{\frac{d-1}{p}-u+r}}{L^p}
    \\&
    =
    \left(\frac{1}{1+\min\{1,\abs{\tau}/\abs{\xi}\}\abs{b}}\right)^r
    \left(\frac{2}{(u-r)p-d}\right)^\frac{1}{p}
\end{align*}
For a single partial derivative, this leads to
\begin{align*}
	\norm{
		\partial^\alpha
		g(\cdot;\xi,b)
	}{L^{p}(\eabs{\cdot}^{-u};\rr{d})}
    &\lesssim
      \abs{\xi}^{\ell}
    \eabs{\min\{1,\abs{\tau}/\abs{\xi}\}\abs{b}}^{-r}
\end{align*}
and by counting the number of partial derivatives up to order $\ell$, we get the bound on in the weighted Sobolev norm.
\end{proof}

Finally, we can state the approximation result over unbounded domains.
\begin{theorem}[Approximation over Unbounded Domain]
    \label{thm:unbounded}
    Let $d,\ell,N,m\in\nn{}$ and 
    $p,r,u,v\in\rr{}$ such that
    $2\leq p<\infty$,
    $1<r\leq v$, and
    $(u-r)p>d$.
    Furthermore, let $\varrho\in \sobolev{\ell,\infty}[\eabs{\cdot}^v][\rr{}]$ be an activation function and $\mathbb{D}_{\varrho}$ be the corresponding dictionary over $\rr{d}$.
    For every target function $f\in\barron{\ell+r}$ we get
    \begin{align*}
        \inf_{f_N\in\Sigma_{N}(\mathbb{D}_\varrho)}
        \norm{f-f_N}{\sobolev{\ell,p}[\eabs{\cdot}^{-u}]}
        \lesssim
        N^{-\frac{1}{2}}
        \norm{f}{\barron{\ell+r}}.
    \end{align*}
\end{theorem}

\begin{proof}
    The space in which we measure the error is $\sobolev{\ell,p}[\eabs{\cdot}^{-u}]$.
    For $p=2$ we are dealing with a Hilbert-space, which is by definition a type-2 Banach space.
    For $2\leq p<\infty$, we know that $L^p(\eabs{\cdot}^{-u};\rr{d})$ is a type-2 Banach space (see e.g., \citep{Siegel22SharpBoundsApproximation}).
    This immediately extends to $\sobolev{\ell,p}[\eabs{\cdot}^{-u}]$ by using the equivalence of finite-dimensional $p$-norms.

    The type-2 property allows us to use Maurey's sampling argument (see \citep{Pisier80RemarquesResultatNon,Barron93UniversalApproximationBounds,Siegel22SharpBoundsApproximation}) in combination with the given error norm.
    Additionally, we need to provide an integral representation of the target function in terms of some dictionary.
    We then need to show that the dictionary is bounded in $\sobolev{\ell,p}[\eabs{\cdot}^{-u}]$ and that the target function has finite variation norm for the given dictionary.

    For the integral representation we take a similar approach to \citep{Barron93UniversalApproximationBounds,Siegel20ApproximationRatesNeural,Abdeljawad23SpaceTimeApproximationShallow}.
    The first step in this approach is to express the Fourier basis in terms of the target function.
    To do so, we start with a linear shift in the Fourier transform of the activation function
    \begin{align*}
        \hat{\varrho}(\tau)
        &=\frac{1}{\sqrt{2\pi}}\int_{\rr{}}\varrho(t)e^{-i\tau t}dt
        =\frac{1}{\sqrt{2\pi}}\int_{\rr{}}\varrho(\innerProd{\xi}{x}+b)e^{-i\tau (\innerProd{\xi}{x}+b)}\dee b
    \end{align*}
    which allows us to express the exponential term as
    \begin{align*}
        e^{i\tau\innerProd{\xi}{x}}
        =\frac{1}{\sqrt{2\pi}\hat{\varrho}(\tau)}\int_{\rr{}}\varrho(\innerProd{\xi}{x}+b)e^{-i\tau b}\dee b
    \end{align*}
    under the assumption that $\tau\neq 0$ and $\hat{\varrho}(\tau)\neq 0$.
    As a second step, we insert this into the Fourier transform of the target function, which leads to the representation    
    \begin{align*}
        f(x)
        &=
        \frac{1}{(2\pi)^\frac{d}{2}}\int_{\rr{d}}e^{i\innerProd{\xi}{x}}\hat{f}(\xi)\dee \xi
        \\&
        =\frac{1}{(2\pi)^\frac{d+1}{2}\hat{\varrho}(\tau)}
        \int_{\rr{d}}\int_{\rr{}}
        \varrho\left({\innerProd{\xi}{x}}/{\tau}+b\right)\hat{f}(\xi) e^{-i\tau b}\dee b \dee \xi.
    \end{align*}
    The last step regarding the integral representation is to specify the dictionary $\mathbb{D}$ and the corresponding measure of $f$ in the variation space of $K(\mathbb{D})$.
    In order to have a bounded dictionary and finite variation norm, we extend the representation of $f$ by the weight $\omega$ (as defined in the statement of the theorem) and
    $$
    \vartheta(\xi,b)=\vartheta_1(\abs{\xi},b)=\eabs{b}^r\eabs{\xi}^{-r}\quad \text{with } 1<r\leq v
    $$
    to get
\begin{align*}
	f(x)
	=
	c
	\int_{\rr{d}}
		\int_{\rr{}}
			\frac{\vartheta(\xi,b)}{\eabs{\xi}^{\ell}}\varrho(\innerProd{\xi}{x}/\tau+b)
			\frac{\eabs{\xi}^{\ell}}{\vartheta(\xi,b)}\hat{f}(\xi)e^{-i\tau b}
		\dee b
	\dee\xi.
\end{align*}
This now leads to the dictionary
\begin{align*}
    \mathbb{D}=\left\{
    \frac
        {\vartheta(\xi,b)}
        {\eabs{\xi}^{\ell}}
    \varrho(\innerProd{\xi}{\cdot}/\tau+b)
    :
    \rr{d}\to\rr{}
    \middle|
    (\xi,b)\in\rr{d}\times\rr{}\right\}
\end{align*}
and the measure
\begin{align*}
    \dee \mu_f(\xi,b)
    =
    \frac
        {\eabs{\xi}^{\ell}}
        {\vartheta(\xi,b)}
    \hat{f}(\xi)e^{-i\tau b}
	\dee b
	\dee\xi
\end{align*}

For the variation norm we use H\"{o}lder's inequality for the norm over the weight parameter, which leads to the following bound in terms of the $\FL{1}$-norm:
\begin{subequations}
\label{eq:unbounded:var_norm}
\begin{align*}
	\norm{f}{\mathcal{K}(\mathbb{D})}
	&=
	\norm{\frac{\eabs{\cdot}^{\ell}}{\vartheta(\cdot,\cdot\cdot)}\hat{f}(\cdot)}{L^1(\rr{d}\times\rr{})}
	=
	\norm{
    \eabs{\cdot}^{\ell+r}\hat{f}(\cdot)
	\int_{\rr{}}
		\frac{1}{\eabs{b}^r}
	\dee b
	}{L^1(\rr{d})}
	\\&
	\leq
	\norm{\eabs{\cdot}^{\ell+r}\hat{f}(\cdot)}{L^1(\rr{d})}
	\norm{
		\int_{\rr{}}
			\frac{1}{\eabs{b}^r}
		\dee b
	}{L^\infty(\rr{d})}
    \\&
    =
	\norm{f}{\barron{\ell+r}}
    \frac{2}{r-1}.
\end{align*}
\end{subequations}

The dictionary constant $K_\mathbb{D}$ is defined as
\begin{align*}
	K_\mathbb{D}
	=
	\sup_{h\in\mathbb{D}}\norm{h}{\sobolev{\ell,p}[\eabs{\cdot}^{-u}]}
	=
    \frac{\vartheta(\xi,b)}{\eabs{\xi}^{\ell}}
	\sup_{(\xi,b)\in\rr{d}\times\rr{}}
		\norm{
			\varrho(\innerProd{\xi}{\cdot}/\tau+b)
		}{\sobolev{\ell,p}[\eabs{\cdot}^{-u}]}
\end{align*}
and with the result $(ii)$ from \cref{lem:unbounded:activation}, we get
\begin{align*}
    K_\mathbb{D}
    \lesssim
    \frac{\eabs{b}^r}{\eabs{\xi}^{\ell+r}}
    \abs{\xi}^\ell
    \eabs{\min\{1,\abs{\tau}/\abs{\xi}\}\abs{b}}^{-r}
    \leq
    \frac{\eabs{b}^r}{\eabs{\xi}^{r}
    \eabs{\min\left\{1,\frac{\abs{\tau}}{\abs{\xi}}\right\}\abs{b}}^{r}}
\end{align*}
For small $\xi$ (i.e., $\abs{\xi}<\abs{\tau}$), we can simplify this to
\begin{align*}
    \eabs{\xi}^{-r}
    \eabs{b}^r
    \eabs{\min\{1,\abs{\tau}/\abs{\xi}\}\abs{b}}^{-r}
    =
    \eabs{\xi}^{-r}\eabs{b}^{r-r},
\end{align*}
which uniformly bounded by $1$.
Conversely for large $\xi$ (i.e., $\abs{\xi}\geq\abs{\tau}$) we get 
\begin{align*}
    \eabs{\xi}^{-r}
    \eabs{b}^r
    \eabs{\min\{1,\abs{\tau}/\abs{\xi}\}\abs{b}}^{-r}
    =
    \frac
        {\eabs{b}^r\abs{\xi}^r}
        {\eabs{\xi}^r(\abs{\xi}+\abs{\tau}\abs{b})^r}
    =
    \frac
        {\eabs{b}^r}
        {(\abs{\xi}+\abs{\tau}\abs{b})^r},
\end{align*}
which is monotonically decreasing in $\abs{\xi}$ and $\abs{b}$, thus, uniformly bounded by $\abs{\tau}^{-r}$ (this is obtained by setting $b=0$ and $\abs{\xi}=\abs{\tau}$).

Overall, this results in the following bound on the dictionary:
\begin{align}
\label{eq:unbounded:dictionary_bound}
    K_\mathbb{D}
    \lesssim
    \min\{1,\tau\}^{r}
\end{align}
where the implied constant depends on the number of dimensions $d$, the regularity $\ell$, the integrability $p$, the activation function $\varrho$, and the weight $\omega$.

Thus, by Maurey's sampling argument in the formulation of \cite[(1.17)]{Siegel22SharpBoundsApproximation} with the variation norm being bounded as in \cref{eq:unbounded:var_norm} and the dictionary being bounded as in \cref{eq:unbounded:dictionary_bound}, we have
\begin{align*}
    \inf_{f_N\in\Sigma_{N,M_f}(\mathbb{D})} \norm{f-f_N}{\sobolev{\ell,p}[\eabs{\cdot}^{-u}]}
    \lesssim
    N^{-\frac{1}{2}}\norm{f}{\barron{\ell+r}}
\end{align*}
with $M_f=\norm{f}{\mathcal{K}(\mathbb{D})}$
In this formulation, the dictionary $\mathbb{D}$ is simply a rescaling of $\mathbb{D}_\varrho$, therefore, $\Sigma_{N,M}(\mathbb{D})\subset\Sigma_{N}(\mathbb{D}_\varrho)$ and finally,
\begin{align*}
    \inf_{f_N\in\Sigma_{N}(\mathbb{D}_\varrho)} \norm{f-f_N}{\sobolev{\ell,p}[\eabs{\cdot}^{-u}]}
    &
    \leq
    \inf_{f_N\in\Sigma_{N,M}(\mathbb{D})} \norm{f-f_N}{\sobolev{\ell,p}[\eabs{\cdot}^{-u}]}
    \\&
    \lesssim
    N^{-\frac{1}{2}}\norm{f}{\barron{\ell+r}}.
\end{align*}
\end{proof}

\bibliography{NeuralNetworksApproximationRates}

\appendix

\end{document}